\documentclass[twoside]{article}

\usepackage{graphicx}
\usepackage{amsmath, amsthm, amssymb, amsfonts, bbm, dsfont}
\usepackage{booktabs, caption, tabularx, xcolor, parskip, subcaption}
\usepackage{algorithm, algpseudocode}

\usepackage{xcolor}
\usepackage{soul}

\usepackage[preprint]{aistats2027}

\usepackage[authoryear,round]{natbib}
\usepackage[colorlinks=true, linkcolor=blue]{hyperref}
\usepackage{cleveref}

\newtheorem{theorem}{Theorem}
\newtheorem{lemma}{Lemma}
\newtheorem{corollary}{Corollary}
\newtheorem{assumption}{Assumption}
\newtheorem{specification}{Specification}
\newtheorem{definition}{Definition}

\newcommand{\Dn}{\mathcal{D}_{n}}
\newcommand{\bX}{\textbf{X}}

\newcommand{\bx}{\textbf{x}}

\newcommand{\ind}{\perp\!\!\!\!\perp}
\newcommand{\E}{\mathbb{E}}

\newcolumntype{C}{>{\centering\arraybackslash}X}

\begin{document}

\runningtitle{CausalTreeHFD}

\twocolumn[

\aistatstitle{Explaining Heterogeneous Treatment Effects with Causal Forests and the Hoeffding Functional Decomposition}

\aistatsauthor{ Gaëtan Patinier \And Clément Bénard \And Margaux Törnqvist }

\aistatsaddress{ Theremia \And Thales cortAIx-Labs \And Theremia } ]

\begin{abstract}

The estimation of heterogeneous treatment effects is a central objective of causal inference. However, state-of-the-art algorithms to solve this task are based on machine learning principles, and thus behave as black boxes. Consequently, the identification of covariates driving heterogeneity remains a challenging and open problem. To address this limitation, we introduce CausalTreeHFD, an interpretable extension of causal forests, based on the Hoeffding functional decomposition adapted to tree ensembles. Our approach decomposes predictions of conditional average treatment effects (CATE) into contributions from individual covariates and pairwise interactions, providing a direct characterization of the factors driving treatment effect heterogeneity. We establish the consistency of CausalTreeHFD, which converges towards the Hoeffding decomposition of the underlying CATE function, under standard assumptions for the analysis of causal forests. Experiments on simulated and real datasets show that CausalTreeHFD accurately identifies relevant covariates and interactions, and precisely recovers their contributions to heterogeneous treatment effects.
    
\end{abstract}

\section{INTRODUCTION}\label{sec:intro}

The identification of factors driving heterogeneity in treatment effects is essential for the interpretation of causal inference models in high-stakes settings. For example, in the medical sector, these insights may help to identify patient subgroups, which are the most likely to benefit from a given treatment, and further guide research on specific subtypes of diseases \citep{PATH}. A wide range of methods have been recently introduced to estimate heterogeneous treatment effects (CATE), among which causal forests have become a prominent approach \citep{athey2018generalizedrandomforests, doi:10.1177/20531680231153080, https://doi.org/10.1111/insr.12610}. Despite their strong theoretical guarantees, causal forests are considered as black boxes, since they only provide limited insights into the covariates and interactions, which mostly contribute to their CATE estimates. Improving the interpretability of causal forests is therefore essential to analyze treatment effect heterogeneity, support reliable decision-making, and generate hypotheses about mechanisms of treatment responses.
Although the explainability of machine learning algorithms has raised strong interest in the community over the past decade, most efforts have been dedicated to standard regression and classification problems, leaving causal inference models with few explainability methods.
Our main objective is therefore to contribute to filling this gap, by developing new methods to interpret causal forests.

\subsection{Related Work}

Existing approaches to interpret heterogeneous treatment effect models can be broadly divided into two categories: model-specific methods for causal forests, and model-agnostic explainability methods.

The \textit{grf} package provides several tools to interpret causal forest predictions \citep{athey2018generalizedrandomforests}. Its variable importance measure ranks covariates according to the frequency and depth of their splits within the forest. While computationally inexpensive, such split-based importance measures can favor continuous and categorical covariates with many levels, and underestimate the importance of binary covariates \citep{Gromping01112009}. The \textit{grf} package also implements the best linear projection (BLP), which summarizes treatment effect heterogeneity through a linear projection of doubly robust pseudo-outcomes onto the covariates, useful for estimating linear trends along each covariate. Finally, policy trees identify subgroups with different estimated treatment effects using a decision tree \citep{Sverdrup2020}. Other methods compute variable importance measures by comparing the difference in accuracy between a full model and a model excluding a single covariate \citep{Hines_2025}.

A second line of research considers model-agnostic explainability methods. Among these, SHAP values (SHapley Additive exPlanations) have become one of the most widely used frameworks to attribute predictions to individual variables \citep{NIPS2017_8a20a862, Muschalik_2024, yu2023lineartreeshap}. Its popularity stems from its broad applicability across model classes and the availability of efficient methods. In particular, TreeSHAP \citep{lundberg2020local} is a fast algorithm to compute SHAP values for tree ensembles, and is often used to identify important covariates. However, TreeSHAP cannot be directly computed for causal forests, since their predictions do not take the appropriate form of the sum of single tree predictions. Additionally, SHAP values can become difficult to interpret when covariates are strongly correlated, and the estimated quantities critically rely on the choice of value function and feature distribution, leading to multiple competing notions of feature attributions \citep{sundararajan2020shapleyvaluesmodelexplanation}.

Similarly, the Hoeffding functional decomposition (HFD) quantifies the contributions of covariates to a prediction function through the decomposition of the function into a sum of main effects and interaction terms, defined with respect to the joint distribution of the input covariates.
The HFD is unique and often highly sparse, and therefore provides interpretable representations of multivariate functions.
Unlike SHAP values, the HFD naturally accounts for feature dependence through hierarchical orthogonality constraints.
It is well-known that the HFD is difficult to estimate from a data sample when input covariates are dependent \citep{Hooker01092007, chastaing2012hoeffding}. However, the TreeHFD algorithm was recently introduced to efficiently estimate the HFD of tree ensembles \citep{benard2025treehfd}, suggesting a natural application to causal forests. Nevertheless, TreeHFD also requires an additive form of tree ensembles, which is not the case of causal forests, as already mentioned for TreeSHAP. We will circumvent this limitation as explained in the following subsection. %

\subsection{Contributions}

The goal of the article is to introduce new methods to explain causal forests, using the Hoeffding functional decomposition to break down CATE predictions into contributions from individual covariates and their interactions. Our approach provides faithful and interpretable representations of heterogeneous treatment effects.

More precisely, we introduce two methods to estimate the HFD of heterogeneous treatment effects, based on causal forests.
We first show that gradient-boosting surrogate models can accurately reproduce causal forests predictions, and can then be turned into functional decompositions through TreeHFD. 
Next, we introduce CausalTreeHFD, based on a new estimate derived from causal forests, which takes the form of an additive ensemble of tree predictions. Importantly, we show the consistency of CausalTreeHFD with respect to the theoretical HFD of treatment effects. In particular, we show that covariates that do not contribute to treatment effect heterogeneity have vanishing HFD components asymptotically, ensuring that the decomposition faithfully identifies the variables driving heterogeneity.
Finally, we empirically evaluate these two methods with simulated and real-world datasets, to demonstrate their ability to recover relevant prognostic factors and quantify their contributions to heterogeneous treatment effects, and to show that the introduced methods outperform the existing BLP baseline.

\section{ALGORITHMS} \label{sec:methods}

We first provide the required background definitions, and then introduce two methods to estimate the HFD of treatment effects with causal forests: gradient-boosting surrogate models and CausalTreeHFD.

\subsection{Background Definitions}

We formalize the causal inference framework used throughout the article, and recall the causal forest algorithm and the HFD for tree ensembles.

\subsubsection{Causal Inference Framework}

We adopt the potential outcomes framework of Rubin \citep{rubin}. Hence, we consider the random variables $(\bX, W, Y(0), Y(1))$, where $\bX \in \mathcal{X} \subseteq \mathbb{R}^p$ denotes the \(p \in \mathbb{N}^{\star}\) covariates describing individual characteristics, \(W \in \{0, 1\}\) the treatment assignment, and \(Y(0)\) and \(Y(1)\) denote the potential outcomes under control ($W=0$) and treatment ($W=1$) respectively.
We consider the standard model of partially linear potential outcome, defined by
\begin{equation*}
    \forall w \in \{0, 1\}, \quad Y(w) = m_0(\bX) + w \cdot \tau(\bX) + \varepsilon(w),
\end{equation*}
where \(\mathbb{E}[\varepsilon(0) \mid \bX] = \mathbb{E}[\varepsilon(1) \mid \bX] = 0\). 
Throughout, we write \(\mathcal{S}\) for the power set of \(\{1, \dots, p\}\), and for \(J \in \mathcal{S}\) we denote by \(\bX^{(J)}\) the subvector of \(\bX\) with components in \(J\). Additionally, we assume that the CATE depends only on a subset \(\mathcal{H} \in \mathcal{S}\) of covariates, which we call the set of treatment effect modifiers, so that \(\tau(\bX) = \tau(\bX^{(\mathcal{H})})\). Our objective is then to estimate the CATE, which can also be expressed as 
\[
\tau(\bX^{\mathcal{(H)}}) = \mathbb{E}[Y(1) - Y(0) \mid \bX],
\]
using the i.i.d. sample $\mathcal{D}_n = \{(\bX_i, W_i, Y_i)\}_{i=1}^n$.

Identification of the CATE relies on the standard assumptions of consistency, positivity and unconfoundedness, as formalized below.

\begin{assumption} \label{asm:causal}
Consistency, positivity, and unconfoundedness are satisfied, that is:

(i) \textbf{Consistency.} The observed outcome is $Y = W Y(1) + (1-W)Y(0)$,
which means that the potential outcome under the assigned treatment equals the observed outcome.

(ii) \textbf{Positivity.} For $\bx \in \mathcal{X}$, let \(e(\bx) = \mathbb{P}(W= 1 \mid \bX=\bx)\) denote the propensity score. There exists \(\eta > 0\) such that for all $\bx \in \mathcal{X}$, $\eta < e(\bx) < 1 - \eta$.

(iii) \textbf{Unconfoundedness.} Potential outcomes and treatment assignment are independent conditional on the covariates: $(Y(0), Y(1)) \ind W \mid \bX$.
\end{assumption}

\subsubsection{Causal Forests}

Causal forests estimate the CATE $\tau(\bX^{(\mathcal{H})})$ by solving, at each target point $\bx \in \mathcal{X}$, the local moment equation
\begin{align*} %
    \tau(\bx^{(\mathcal{H})}) \cdot \operatorname{Var}(W \mid \bX = \bx) - \operatorname{Cov}(W, Y \mid \bX  = \bx) = 0.
\end{align*}
To solve this empirically, \textit{grf} builds an adaptive neighborhood around $\bx$ from the training data $\mathcal{D}_n$. Each tree $b \in \{ 1, \dots, B\}$ contributes weights
\[
    \alpha_{ib}(\bx) = \frac{\mathbf{1}\{\bX_i\in L_b(\bx)\}}{n_b(\bx)},
\]
where $L_b(\bx)$ is the leaf of tree $b$ containing $\bx$, and $n_b(\bx)$ is the number of observations of the estimation subsample of tree $b$ in $L_b(\bx)$. The weights $\alpha_{ib}$ are set to $0$ for observations $\bX_i$ outside of the estimation subsample, built by honesty. Averaging over the forest gives the aggregate similarity weights $\smash{\alpha_i(\bx) = \sum_{b=1}^B \alpha_{ib}(\bx) / B}$,
which act as a kernel, giving more weight to observations landing in the same leaf as $\bx$. Writing the weighted local means as $\bar{W_\alpha} = \sum_{i=1}^n \alpha_i(\bx)W_i$ and $\bar{Y_\alpha} = \sum_{i=1}^n \alpha_i(\bx)Y_i$, the local variance and covariance can be estimated by %
\[
\sum_{i=1}^n \alpha_{i}(\bx) (W_i - \bar{W_\alpha})^2, \quad
\sum_{i=1}^n \alpha_{i}(\bx) (W_i - \bar{W_\alpha}) (Y_i - \bar{Y_\alpha}).
\]
Plugging these into the moment equation gives the closed-form \textit{grf} estimator
\begin{equation}\label{eq:grf_formula}
    \tau_{\mathrm{grf}}(\bx) = \frac{\sum_{i=1}^n \alpha_{i}(\bx) (W_i - \bar{W_\alpha})(Y_i - \bar{Y_\alpha})}{\sum_{i=1}^n \alpha_{i}(\bx) (W_i - \bar{W_\alpha})^2}.
\end{equation}
In practice, $W_i$ and $Y_i$ are first residualized against $\bX_i$ using regression forests, to remove confounding effects.

\subsubsection{Hoeffding Functional Decomposition}

The HFD provides an additive decomposition of a multivariate function, with a component associated with each individual variable and possible interaction. 
To formalize the HFD, we need the following assumption about the bounded covariate density, and the covariate support $\mathcal{X}$ assumed to be a hyperrectangle, as it is also standard in the analysis of random forests \citep{scornet2015consistency, wager2017estimationinferenceheterogeneoustreatment}.
This assumption is used by \citet{Stone1994TheUO}, \citet{Hooker01092007}, and \citet{chastaing2012hoeffding} to state the existence and uniqueness of the HFD, as recalled in the theorem below, in the case of treatment effects.
\begin{assumption}\label{asm:three}
     The treatment effect $\tau$ is square-integrable. The random vector \(\mathbf{X}\) takes values in a hyperrectangle \(\mathcal{X} \subseteq \mathbb{R}^p\), and admits a density $f$ with respect to the Lebesgue measure, which is Lipschitz and bounded above and below by strictly positive constants.
\end{assumption}

\begin{theorem}[HFD] \label{thm:hfd}
    If Assumption \ref{asm:three} is satisfied, then there is a unique set of functions $\{\tau^{(J)}\}_{J \in \mathcal{S}}$, such that for all $I \subsetneq J \in \mathcal{S}$, $\mathbb{E}[\tau^{(J)}(\bX^{(J)}) \mid \bX^{(I)}] = 0$, and
    \begin{align*}
        \tau(\bX) = \sum_{J \in \mathcal{S}} \tau^{(J)}(\bX^{(J)}).
    \end{align*}
\end{theorem}

The term \(\tau^{(\emptyset)}\) is a constant, while first-order components \(\tau^{(j)}\) capture contributions of individual covariates, and higher-order components capture interactions.
The hierarchical orthogonality conditions stated in Theorem \ref{thm:hfd} ensure that the captured interaction term for a subset $J \subseteq \mathcal{H}$ represents variations of $\tau$ that cannot be explained by any of the covariates or strict subsets of covariates of \(J\). Consequently, the orthogonality conditions also enforce uniqueness of the decomposition, and often lead to a high sparsity. Because of these properties, the HFD provides powerful representations to interpret multivariate functions.

As an example, consider \(\tau(\bX) = X^{(1)} + X^{(1)} X^{(2)}\), where $\textbf{X}=(X^{(1)}, X^{(2)})$ is bivariate Gaussian with unit variances and correlation $\rho$. When $\rho = 0$, we directly have $\tau^{(1)} = X^{(1)}$, $\tau^{(2)} = 0$ and $\tau^{(1, 2)} = X^{(1)} X^{(2)}$. When $\rho \neq 0$, a main effect $\tau^{(2)}(X^{(2)}) = \rho ((X^{(2)})^2 - 1) / (1 + \rho^2)$ appears and is subtracted from $\tau^{(1, 2)}$, together with its symmetric counterpart for $X^{(1)}$---see Table \ref{tab:hfd_B} in Appendix \ref{app:exp}. Indeed, correlated covariates share redundant information, so either one alone already explains part of the interaction; removing these effects ``purifies'' the interaction \citep{lengerich2020purifying} and enforces hierarchical orthogonality. The variance of $\tau^{(1, 2)}$ therefore decreases with $\rho$ and vanishes as $\rho \to 1$, when the two covariates become fully redundant and no interaction term is needed.

While the HFD has several attractive properties, computing its components can be challenging when covariates are dependent, which is almost systematically the case in practice. As mentioned in the introduction, the TreeHFD algorithm was recently developed to estimate the HFD of tree-based models from a data sample \citep{benard2025treehfd}. This method possesses desirable theoretical properties, encouraging its use.
The key idea is to exploit the piecewise-constant structure of decision trees by constructing a Cartesian partition of the hyperrectangle $\mathcal{X}$ induced by the tree splits, as illustrated in Figure \ref{fig:overall_figure}, and formalized in Definition \ref{def:ctp} below, with $\mathcal{S}$ the power set of $\{1, \hdots, p\}$.
The TreeHFD algorithm consists in computing this partition for each tree, and by considering the HFD of the tree with respect to the covariate density averaged over each cell of this partition. With this discretized density, each component of the HFD becomes a piecewise constant function, where each cell of the partition defines a constant region. Then, the decomposition can be parametrized, and finally estimated by the minimization of an appropriate loss function to enforce the orthogonality constraints, as detailed below. %

\begin{definition}[Cartesian tree partitions] \label{def:ctp}
    Under Assumption \ref{asm:three} and given a tree $\tau_{n,b}$ of a causal forest, for each variable $j \in \{1, \dots, p\}$, we denote by $\smash{\mathcal{A}_b^{(j)}}$ the sequence of intervals induced by the splits of the tree along dimension $j$. The Cartesian tree partition of $\tau_{n, b}$ along dimensions $J \in \mathcal{S}$ is defined as the Cartesian product of these sets of intervals: $\smash{\mathcal{A}_b^{(J)} = \bigotimes_{j \in J} \mathcal{A}_b^{(j)}}$.
\end{definition}

\begin{figure}[t]
    \centering
    
    \begin{subfigure}[b]{0.23\textwidth}
        \centering
        \includegraphics[width=\textwidth]{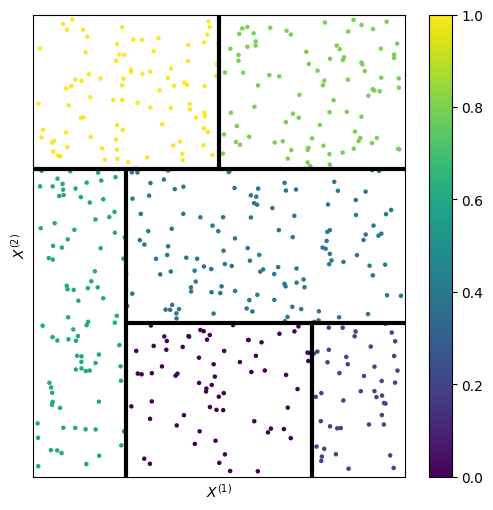}
        \caption{Decision tree splits.}
        \label{fig:subfig_a}
    \end{subfigure}
    \hfill
    \begin{subfigure}[b]{0.23\textwidth}
        \centering
        \includegraphics[width=\textwidth]{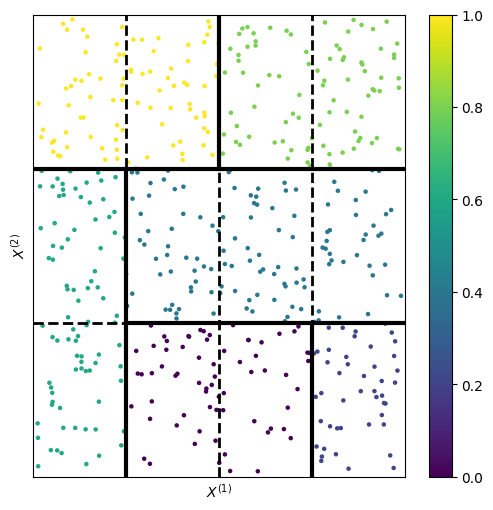}
        \caption{Cartesian partition.}
        \label{fig:subfig_b}
    \end{subfigure}
    
    \caption{The splits of a decision tree along the first two dimensions are extended to define a Cartesian tree partition, along which the TreeHFD component \(\{1, 2\}\) is constant.}
    \label{fig:overall_figure}
\end{figure}

\subsection{Gradient-boosting Surrogate Models}\label{subsec:algorithms}

The standard causal forest $\tau_{\mathrm{grf}}$ first aggregates the weights from all trees and then solves the local moment equation, as formalized in Equation (\ref{eq:grf_formula}) above. Since the aggregation occurs before forming the ratio, $\tau_{\mathrm{grf}}$ cannot be written as an additive ensemble of tree predictions, which prevents the direct application of TreeHFD.
A natural approach to overcome this limitation is to reproduce causal forest predictions using surrogate models built with standard tree ensembles for regression. Therefore, we propose the following procedure, where a causal forest is first fit using the training data $\Dn$. Then, a gradient-boosting model is fit on the generated dataset $\{ (\bX_i, \tau_{\mathrm{grf}}(\bX_i)) \}_{i=1}^n$, and finally TreeHFD is applied to the obtained model to provide the HFD of the treatment effect.
While this method exhibits good performance as shown in the experimental section, it requires building a second tree ensemble, which can be avoided as we will see.

\subsection{CausalTreeHFD}

CausalTreeHFD avoids the distillation of the first approach and proceeds in two steps. First, we construct an additive tree-level estimate of the CATE from the causal forest, that we call the additive causal forest. Second, we apply TreeHFD to this additive representation to obtain a HFD of the CATE. The complete procedure is summarized in \Cref{algo}.

\paragraph{Step 1: Additive Causal Forests.}

To recover an additive tree structure, we define an alternative estimate of the CATE by solving the local moment equation within each tree separately, and averaging the results of each tree. Hence, we define for each tree $b$, %
\begin{equation}\label{eq:taunb}
    \tau_{n,b}(\bx) = \frac{\sum_{i=1}^n \alpha_{ib}(\bx) (W_i - \bar{W}_{\alpha b})(Y_i - \bar{Y}_{\alpha b})}{\sum_{i=1}^n \alpha_{ib}(\bx) (W_i - \bar{W}_{\alpha b})^2}.
\end{equation}
where $\bar{W}_{\alpha b} = \sum_i \alpha_{ib}(\bx) W_i$ and $\bar{Y}_{\alpha b} = \sum_i \alpha_{ib}(\bx) Y_i$ are the local means estimated using only the weights of tree $b$. Hence, each \(\tau_{n, b}(\bx)\) is a separate estimate of the CATE. 
The terminal leaves with a number of observations smaller than a parameter \textit{prune.min.sample} are pruned, with a typical default value of $20$. In practice, this pruning step combined with the initial local centering ensures that the denominator of $\tau_{n,b}(\bx)$ is never null.
We then define the additive causal forest $\tau_n(\bx)$ as the average of these tree estimates, given by $\smash{\tau_n(\bx) = \sum_{b=1}^B \tau_{n, b} (\bx) / B}$.
We will show in \Cref{sec:results} that the additive causal forest \(\tau_n(\bx)\) is consistent.%
\paragraph{Step 2: TreeHFD.}

TreeHFD is then applied to each tree of the ensemble to compute its HFD. Concretely, TreeHFD algorithm reformulates the HFD as the solution of a least squares problem, where the optimized parameters are the values of the decomposition components in each cell of the Cartesian tree partitions \( \{\mathcal{A}_b^{(J)} \}_{J \in \mathcal{S}}\).
More precisely, we denote by $\smash{\mu_{b}^{(J)}}$ the piecewise-constant function over the partition $\smash{\mathcal{A}_b^{(J)}}$, where the value of $\smash{\mu_{b}^{(J)}}$ in each cell is parametrized by a coefficient. Then, CausalTreeHFD is fit using the following cost function for each tree
\begin{align*}
    \mathcal{R}^{\star}_b (\{ &\mu_{b}^{(J)} \}_{J \in \mathcal{S}})
    = \E \Big[\Big(\tau_{n, b}(\bX) - \sum_{J \in \mathcal{S}} \mu_{b}^{(J)}(\bX^{(J)})\Big)^2\Big]
      \\ &+ \sum_{J \in \mathcal{S}} \sum_{j \in J}  \sum_{A \in \mathcal{A}_b^{(J \setminus j)}} \E[\mu_{b}^{(J)}(\bX^{(J)}) \mathds{1}_{\{\bX^{(J \setminus j)} \in A\}} ]^2,
\end{align*}
where the cost function minimized in practice is the empirical counterpart $\mathcal{R}_{n,b}$ of $\mathcal{R}^{\star}_b$, with the expectation estimated through the empirical densities $f_{n,b}$ of the inputs $\bX$, computed over the Cartesian tree partitions with $\Dn$---see Appendix \ref{app:proof_thm5}. Hence we have
\begin{align*}
    \smash{\{\tau_{n, b}^{(J)}\}_{J \in \mathcal{S}}} = \underset{\{ \mu_{b}^{(J)}\}}{\mathrm{argmin}} \hspace*{1mm} \mathcal{R}_{n,b}(\{ \mu_{b}^{(J)}\}_{J \in \mathcal{S}}).
\end{align*}
In practice, only the variable subsets that occur in the tree paths and of size at most $2$ are selected in $\mathcal{R}_{n,b}$, which greatly limits the number of terms, and makes the least squares problems tractable---see Appendix \ref{app:complexity} for the computational complexity. Excluding interactions of order three and higher very often has a negligible impact, as we will see in Section \ref{sec:experiments}.

Since the decomposition is additive, the HFD components of the forest are obtained by averaging the corresponding components across all trees for $J \in \mathcal{S}$,
\[
    \tau_n^{(J)}(\bX^{(J)}) = \frac{1}{B} \sum_{b=1}^B \tau_{n, b}^{(J)}(\bX^{(J)}).
\]
The resulting functions \(\{\tau_n^{(J)}\}_{J \in \mathcal{S}}\) provide an interpretable decomposition of the estimated CATE. In particular, \(\tau_n^{(j)}\) quantifies the contribution of covariate \(\bX^{(j)}\) to treatment effect heterogeneity, while higher-order terms characterize covariate interactions.

\begin{algorithm}[t]
\caption{CausalTreeHFD}\label{algo}

\begin{algorithmic} [1]
\Require Covariates matrix \(\bX\), treatment assignment \(W\) and outcome \(Y\).
\State $\tau_{\mathrm{grf}} \leftarrow \mathtt{CausalForest}(\bX, W, Y)$
\State $\tau_n \leftarrow \mathtt{Additive CausalForest}(\tau_{\mathrm{grf}}, \bX, W, Y)$
\State $(\tau_n^{(J)})_{J \in \mathcal{S}} \leftarrow \mathtt{TreeHFD}(\tau_n, \bX)$

\State \Return $(\tau_n^{(J)})_{J \in \mathcal{S}}$

\end{algorithmic}
\end{algorithm}

\section{THEORETICAL PROPERTIES}\label{sec:results}
We establish the consistency of CausalTreeHFD under standard regularity conditions on the data-generating process, together with standard constraints for the construction of causal forests. 
The first step of our analysis is to show that additive causal forests are consistent, as proved in the first subsection.

\subsection{Consistency of Additive Causal Forests}

We first state the required assumptions and constraints for the construction of causal forests before presenting the convergence result.
\begin{assumption}[Lipschitz conditional mean functions]\label{asm:one}
The mean functions $m_w(\bx) = \E[Y(w) \mid \bX = \bx]$ for $ w\in\{0,1\}$, and the propensity score $e(\bx)=\mathbb P(W=1\mid \mathbf{X}=\mathbf{x})$ are Lipschitz continuous on \(\mathcal X\).
\end{assumption}

\begin{assumption}[Uniformly bounded conditional second moments]\label{asm:two}
There exists a constant \(C>0\) such that for all $\bx \in \mathcal{X}$, 
$\E[Y^2 \mid \mathbf{X}=\mathbf{x}]\le C$.
\end{assumption}

\begin{specification}\label{spec:one}
    The causal forest is constructed using balanced random splits and honest trees. Each feature has a probability bounded away from zero of being selected at every split. Furthermore, the subsample size \(s_n\) satisfies \(s_n \to \infty\), while the minimum leaf size satisfies \(k_n \to \infty\) and $k_n / s_n \to 0$.
\end{specification}

\begin{theorem}\label{th:one}
    Under Assumptions~\ref{asm:causal}--\ref{asm:two} and \Cref{spec:one}, the additive causal forest is consistent, that is
    \begin{align*}
        \tau_n(\mathbf{x}) \xrightarrow[n \to \infty]{p} \tau(\mathbf{x}).
    \end{align*} 
\end{theorem}

\Cref{th:one} establishes that additive causal forests preserve the consistency guarantees of causal forests. The proof can be found in Appendix \ref{app:proofs}, and follows from the consistency analysis of generalized random forests by showing that the adaptive neighborhoods induced by the forest shrink to the target point, while the effective sample size within each neighborhood diverges, because the trees are not fully grown according to Specification \ref{spec:one}.

\subsection{Consistency of CausalTreeHFD}

We now focus on the theoretical analysis of the HFD for causal forests. We first present the HFD adapted to causal forests in Theorem \ref{th:two} below, where hierarchical orthogonality constraints are discretized over the tree partitions. Notice that the discretization of the orthogonality constraints is equivalent to the averaging of the covariate density over each cell of the Cartesian tree partition to compute the HFD, which was mentioned above. Under such discretization, there is a unique decomposition with all components piecewise constant over the tree partitions.
The proof follows from Theorem $2$ of TreeHFD article for regression forests \citep{benard2025treehfd}.
\begin{theorem} \label{th:two}
    Let Assumption \ref{asm:three} be satisfied and let $\tau_n = \frac{1}{B} \sum_{b} \tau_{n, b}$ be the additive causal forest of Cartesian tree partitions $\smash{\{\mathcal{A}_b^{(J)}\}_{1 \leq b \leq B, J  \in \mathcal{S}}}$.    
    Then, there exists a unique collection of functions $\smash{\{\tau_{n, b}^{\star (J)}\}_{J \in \mathcal{S}, 1 \leq b \leq B}}$, such that for every tree \(b\), \(\smash{\tau_{n, b}^{\star(J)}}\) is piecewise constant over the partition $\smash{\mathcal{A}_b^{(J)}}$, and satisfies for every \(I \subsetneq J\) and \(A \in \mathcal{A}_b^{(I)}\), $\mathbb{E}[\tau_{n, b}^{\star(J)}(\bX^{(J)}) \mid \Dn, \bX^{(I)} \in A] = 0$, and
    \begin{align*}
        \tau_n = \sum_{J \in \mathcal{S}} \tau_n^{\star (J)}
        \hspace*{1mm} \mathrm{with} \hspace*{1mm} \tau_n^{\star (J)} =  \sum_{b=1}^B \frac{\tau_{n, b}^{\star(J)}}{B}, \tau_{n, b} = \sum_{J \in \mathcal{S}} \tau_{n, b}^{\star(J)}.
    \end{align*}
\end{theorem}
Importantly, the decomposition $\{\tau_{n, b}^{\star (J)}\}_{J \in \mathcal{S}, 1 \leq b \leq B}$ of Theorem \ref{th:two} is not the empirical quantity computed by CausalTreeHFD, but the minimizer of $\mathcal{R}_b^{\star}$ for each tree, that is the theoretical decomposition of the fitted causal forest with respect to the unknown discretized distribution of the covariates. As explained above, CausalTreeHFD uses the data sample $\Dn$ to estimate the decomposition in practice, by minimizing $\mathcal{R}_{n,b}$.

Theorem \ref{th:three} below shows the consistency of the decomposition $\smash{\{\tau_{n, b}^{\star (J)}\}_{J \in \mathcal{S}, 1 \leq b \leq B}}$ with respect to the true HFD of the treatment effect, stated in Theorem \ref{thm:hfd}. This means that when the size of $\Dn$ used to fit the causal forest grows, the HFD decomposition with respect to the discretized distribution gets arbitrarily close to the true HFD. This result is critical to show the efficiency of the discretization method to compute the HFD for causal forests. The proof directly combines our Theorem \ref{th:one} and Corollary $1$ of TreeHFD article.
\begin{theorem}\label{th:three}
    If Assumptions~\ref{asm:causal}--\ref{asm:two} and \Cref{spec:one} are satisfied, and trees are bounded, then the discretized HFD of the fitted additive causal forest is consistent, that is, for all $J \in \mathcal{S}$, %
    \[
       \mathbb{E}[(\tau_n^{\star (J)}(\bX^{(J)}) - \tau^{(J)}(\bX^{(J)}))^2] \xrightarrow[n \to \infty]{} 0.
    \]
\end{theorem}

Our main results provide theoretical properties of the decomposition estimated by CausalTreeHFD. The above Theorems \ref{th:two} and \ref{th:three} are derived from the theory developed in TreeHFD article, which does not provide the convergence of empirical estimates output by TreeHFD to the true HFD for regression problems.
In fact, it is possible to further deepen the theoretical analysis of CausalTreeHFD using the following key property: TreeHFD computes the HFD of each tree by minimizing a cost function, and the minimum of this cost function is exactly the Hoeffding decomposition of the tree with respect to the empirical density estimated with the Cartesian tree partition, provided that the empirical density is bounded, as enforced by the following Specification \ref{spec:causaltreehfd}. Therefore, we can state a finite sample bound on the quadratic error of CausalTreeHFD estimates $\smash{\{\tau_n^{(J)}\}}$, using
\begin{align*}
    \Delta_{n, f} = \sup_{1 \leq b \leq B} \sup_{\bx \in \mathcal{X}} |f(\bx) - f_{n,b}(\bx)|,
\end{align*}
which quantifies the error of the empirical density $f_{n,b}$ with respect to the true density $f$ of $\bX$---see the proof in Appendix \ref{app:proof_thm5}.

\begin{specification}\label{spec:causaltreehfd}
    For each tree of the additive causal forest, the empirical density computed over the Cartesian tree partition is truncated from above and below by positive constants and normalized.
    Trees are also truncated by an arbitrarily large constant, and CausalTreeHFD loss is built with variable subsets of all sizes and with the truncated empirical density for each tree.
\end{specification}

\begin{theorem} \label{thm:causaltreehfd_bound}
    Let Assumption \ref{asm:three} and 
    Specification \ref{spec:causaltreehfd} be satisfied.
    Then, there exist two constants $K_1, K_2 > 0$, such that for any causal forest $\tau_n = 1 / B \sum_{b} \tau_{n, b}$, the CausalTreeHFD estimate of $\tau_n$ satisfies for all $J \in \mathcal{S}$, 
    \begin{align*}
        \E[(\tau_n^{(J)}(\bX^{(J)}&) - \tau^{(J)}(\bX^{(J)}))^2] \leq K_1 \E[( \tau(\bX) - \tau_n(\bX) )^2] \\ &+ K_2 \E\Big[ \Delta_{n, f}^2 \Big(\frac{1}{B} \sum_{b=1}^B \sqrt{\E[\tau_{n, b}(\bX)^2|\Dn]}\Big)^2 \Big].
    \end{align*}
\end{theorem}

This bound on the quadratic error of CausalTreeHFD depends on two terms: the error of the additive causal forest to estimate the true CATE, and the error of the empirical distributions, quantified by $\Delta_{n,f}$.
The following corollary states the consistency of CausalTreeHFD, since Theorem \ref{th:one} about the additive causal forest consistency ensures that the first term vanishes, and we assume that $\E[\Delta_{n, f}] \to 0$ for the second term. Notice that this assumption essentially requires a pruning of the tree partitions to enforce that the number of observations in each cell of the Cartesian tree partitions grows to infinity and that the cell diameters vanish, to enforce a high accuracy of the empirical distributions.
\begin{corollary} \label{cor:convergence}
    If Assumptions~\ref{asm:causal}--\ref{asm:two} and 
    Specifications~\ref{spec:one}--\ref{spec:causaltreehfd} are satisfied, and $\E[\Delta_{n, f}] \to 0$, then we have for $J \in \mathcal{S}$
    \begin{align*}
        \E[(\tau_n^{(J)}(\bX^{(J)}) - \tau^{(J)}(\bX^{(J)}))^2] \xrightarrow[n \to \infty]{} 0.
    \end{align*}
    In particular, if $j \notin \mathcal{H}$, then for all $J \in \mathcal{S}$ such that $j \in J$, we have
    \begin{align*}
        \E[\tau_n^{(J)}(\bX^{(J)})^2] \xrightarrow[n \to \infty]{} 0.
    \end{align*}
\end{corollary}
The second part of Corollary \ref{cor:convergence} highlights that the covariates that are not involved in the heterogeneity of the treatment effect can be identified by CausalTreeHFD, since the estimated components converge to the null function. The proof is also in Appendix \ref{app:proof_thm5}.

\section{EXPERIMENTS}\label{sec:experiments}

We evaluate the methods through both the ability to reconstruct the CATE predictions of a causal forest and the ability to recover the variables driving treatment effect heterogeneity.
In our experiments, we compare four methods: the gradient-boosting surrogate model denoted by $\tau_{\mathrm{xgb}}$, CausalTreeHFD denoted by $\tau_n$, the TreeSHAP decomposition of the additive causal forest $\tau_{\mathrm{shap}}$, and the existing BLP baseline $\tau_{\mathrm{BLP}}$. We elaborate on the implementation of these comparative models in Appendix~\ref{app:exp}.

\subsection{Experiment A}
We first consider a synthetic setting in which the treatment effect heterogeneity is driven by a second-order interaction between two covariates.

The covariates are generated independently as \(\bX \sim \mathcal{N}(0, I_8)\) and the data-generating process is defined by
\begin{align*}
     e(\bX) &= 0.4+0.2\cdot\mathbf 1\{X^{(1)}\ge 0\} \\
     m_0(\bX) &= (X^{(3)} X^{(4)})^2 \\
     \tau(\bX) &=  \mathbf{1}\{X^{(1)} \geq 0, X^{(2)} \geq 0\} \\
     Y\mid \bX, W &\sim \mathcal{N}(m_0(\bX) + W\cdot\tau(\bX), 0.5^2).
\end{align*}
Correctly estimating the CATE in this setting is not trivial as the ratio of treatment heterogeneity variance to outcome variance \(\operatorname{Var}[\tau(\bX)] / \operatorname{Var}[Y]\) is less than 3\%. The HFD of the CATE consists of two first-order effects and one second-order interaction, while all remaining components are null. The exact decomposition is given in Appendix~\ref{app:exps}.

We train a causal forest with \(n = 5000\) observations, \(B=500\) trees, across \(50\) Monte Carlo iterations. \Cref{tab:metrics} shows the explained variance on testing data (\(R^2\)) between each algorithm and the causal forest predictions \(\tau_\mathrm{grf}\), as well as the \(R^2\) between each algorithm and the true treatment effect \(\tau\), while \Cref{fig:fullexps} reports the \(R^2\) between the estimated and true non-zero HFD components. The two HFD methods and SHAP strongly outperform the BLP baseline, and exhibit comparable performance across all metrics, with no clear differences in their ability to recover the true HFD components. This is quite expected, since input covariates are independent, and the HFD and SHAP coincide in this setting \citep{bordt2023shapley}. 
\Cref{tab:metrics} also reports a measure of the orthogonality constraints between components of the decompositions---see details in Appendix~\ref{app:metrics}. In this simple independent setting, the three decompositions are almost hierarchically orthogonal.
\begin{figure*}[t]
    \centering
    \includegraphics[width=.9\linewidth]{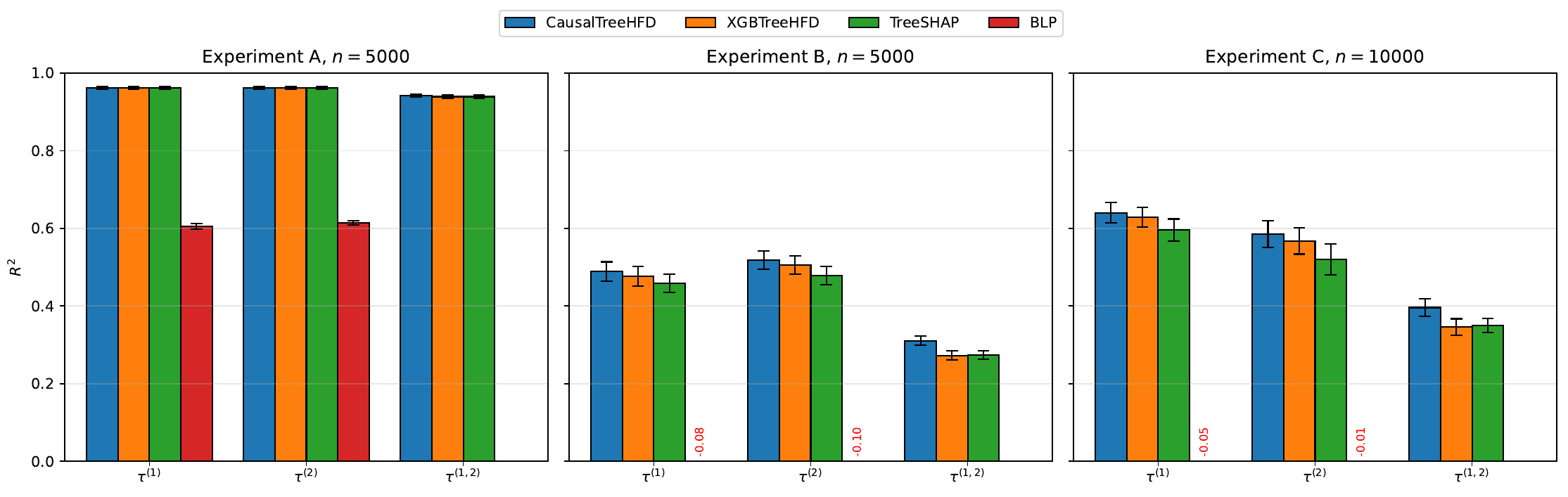}
    \caption{Explained variance and standard errors for each non-null HFD component in Experiments 
    A-C.}
    \label{fig:fullexps}
\end{figure*}

\begin{table}[ht]
\centering
\caption{\(R^2\) with standard errors and orthogonality across estimates, for Experiments A, B, and C.}
\label{tab:metrics}
\small
\setlength{\tabcolsep}{4pt}
\begin{tabular}{@{}ll ccc@{}}
\toprule
\textbf{Metric} & & \textbf{A} & \textbf{B} & \textbf{C} \\
\midrule
$R^2$ vs.\ $\tau_{\mathrm{grf}}$ $\uparrow$ & $\tau_n$ & 0.99\,{\tiny(0.00)} & 0.96\,{\tiny(0.01)} & 0.98\,{\tiny(0.02)} \\
                                 & $\tau_{\mathrm{xgb}}$ & 0.99\,{\tiny(0.00)} & 0.98\,{\tiny(0.00)} & 0.98\,{\tiny(0.02)}\\
                                 & $\tau_{\mathrm{shap}}$ & 0.99\,{\tiny(0.00)} & 0.96\,{\tiny(0.01)} & 0.98\,{\tiny(0.02)}\\
                                 &
                                 $\tau_{BLP}$
                                 & 0.45\,{\tiny(0.00)} &
                                 0.11\,{\tiny(0.02)} & 0.03\,{\tiny(0.00)}\\
\midrule
$R^2$ vs.\ True $\tau$ $\uparrow$ & $\tau_{\mathrm{grf}}$ & 0.91\,{\tiny(0.01)} & 0.40\,{\tiny(0.03)} & 0.61\,{\tiny(0.03)} \\
& $\tau_n$ & 0.91\,{\tiny(0.01)} & 0.41\,{\tiny(0.03)} & 0.62\,{\tiny(0.03)} \\
                                 & $\tau_{\mathrm{xgb}}$ & 0.91\,{\tiny(0.01)} & 0.41\,{\tiny(0.02)} & 0.61\,{\tiny(0.03)}\\
                                 & $\tau_{\mathrm{shap}}$ & 0.91\,{\tiny(0.01)} & 0.40\,{\tiny(0.03)} & 0.61\,{\tiny(0.03)}\\
                                 &
                                 $\tau_{BLP}$
                                 & 0.40\,{\tiny(0.00)} &
                                 -0.05\,{\tiny(0.01)} & -0.01\,{\tiny(0.00)}\\
\midrule
Orthogonality $\downarrow$ & $\tau_n$ & 0.004 & 0.08 & 0.07\\
              & $\tau_{\mathrm{xgb}}$ & 0.003 & 0.08 & 0.07 \\
              & $\tau_{\mathrm{shap}}$ & 0.013 & 0.23 & 0.25 \\
\bottomrule
\end{tabular}
\end{table}

\subsection{Experiment B}
Next, we consider an example in which the covariates driving treatment effect heterogeneity are correlated. Specifically, $\bX \sim \mathcal{N}(0, \Sigma)$, where $\Sigma_{ij}=1$ if $i=j$ and $\Sigma_{ij}=\rho=0.5$ otherwise, for $1 \leq i,j \leq 8$.
The treatment assignment and baseline outcome models are identical to those used in Experiment A, but the CATE is replaced by \(\tau(\bX) = 0.5\cdot\bX^{(1)}\bX^{(2)}\). This choice admits a closed-form HFD under the correlated Gaussian covariate distribution, allowing us to compare the estimated HFD with the ground truth, detailed in Appendix~\ref{app:exps}. The variance of the CATE here is about \(1 \%\) of the total variance of \(Y\).

This setting is harder than experiment A due to the diminished signal and covariate dependence. The three algorithms maintain the same predictive performance as the causal forest, explaining about 40\% of the CATE variance, and strongly outperform the BLP baseline. However, \Cref{fig:fullexps} shows that CausalTreeHFD outperforms the other models when comparing individual HFD components.
Additionally, main effects and interactions are quite strongly correlated for SHAP.

\subsection{Experiment C}

Our last setting repeats Experiment B in higher dimension. The propensity score, baseline outcome and CATE are unchanged, but the covariate dimension is raised to $p = 20$ with the same correlation $\rho = 0.5$, so that each of the $18$ covariates outside $\mathcal{H}$ is correlated with both effect modifiers. We use $n = 10000$ observations and $B = 500$ trees. Since the correlation between $X^{(1)}$ and $X^{(2)}$ is unchanged, the closed-form HFD of Experiment B still applies.

The three methods remain indistinguishable at the level of the CATE, the additive causal forest again matching the accuracy of $\tau_{\mathrm{grf}}$ against the true treatment effect. They separate at the level of the components: \Cref{fig:fullexps} shows CausalTreeHFD ahead on all three non-null components, with the largest margin on the interaction $\tau^{(1,2)}$, where it reaches $R^2 = 0.40$ against $0.35$ for both competitors. Recovering an interaction is precisely the task that degrades when the two effect modifiers compete for splits against $18$ correlated nuisance covariates, and it is where directly decomposing the causal forest is more accurate than distilling it into a surrogate model. The accuracy gap appears for quite large sample size, as expected from the theoretical analysis, where CausalTreeHFD is shown to be consistent, while the limited depth of gradient-boosting surrogate models prevents the error from the distribution discretization to vanish, and no theoretical guarantees exist for TreeSHAP.
Besides, the components of the covariates outside $\mathcal{H}$ meanwhile remain negligible, carrying together $3\%$ of the variance of the non-null components, as predicted by \Cref{cor:convergence}.

\subsection{Real World Data}

We extend our experiments to the analysis of real world data. Although we cannot benchmark model performance on accuracy since we do not know the true treatment effect, we can verify that the outputs of CausalTreeHFD accurately reconstruct the CATE predictions from the causal forest and that the resulting decomposition satisfies the orthogonality constraints, as displayed in \Cref{tab:RWD_metrics}. In addition, we assess the coherence of the resulting interpretations.

\begin{table}[ht]
\centering
\caption{$R^2$ and orthogonality metrics for real data.}
\label{tab:RWD_metrics}
\small
\setlength{\tabcolsep}{8pt}
\begin{tabular}{@{}ll cc@{}}
\toprule
\textbf{Metric} & & \textbf{NHEFS} & \textbf{Welfare} \\
\midrule
$R^2$ vs.\ $\tau_{\mathrm{grf}}$ & $\tau_n$ & 0.88 & 1.00 \\
                                 & $\tau_{\mathrm{xgb}}$ & 0.96 & 1.00 \\
                                 & $\tau_{\mathrm{shap}}$ & 0.87 & 1.00 \\
                                 &
                                 $\tau_{BLP}$
                                 & 0.31 &
                                 0.81 \\
\midrule
Orthogonality & $\tau_n$ & 0.07 & 0.03 \\
              & $\tau_{\mathrm{xgb}}$ & 0.02 & 0.01 \\
              & $\tau_{\mathrm{shap}}$ & 0.13 & 0.59 \\
\bottomrule
\end{tabular}
\end{table}

\paragraph{NHEFS.}
We analyze the \textbf{NHEFS} data on weight gain following smoking cessation \citep{madans1986}, where treatment is smoking cessation between 1971 and 1982 and the outcome is weight change over that period. The metrics in \Cref{tab:RWD_metrics} show that all methods accurately represent the causal forest estimates while satisfying orthogonality constraints, except BLP, which exhibits a poor accuracy. The decomposition of CausalTreeHFD is dominated by the weight in 1971, since its component variance is $51\%$ of the treatment effect variance, followed by the duration of smoking ($12\%$), height ($10\%$), smoking intensity ($8\%$) and age ($7\%$). %
These covariates are also ranked highest by the native variable importance of the \textit{grf} package, as reported in \Cref{tab:nhefs}. However, this native method does not provide a precise quantification of the covariate impacts on the treatment effect, and clearly underestimates the strong influence of the main covariate in this case.

\paragraph{GSS Welfare.}

We also analyzed the Welfare data from a General Social Survey experiment \citep{greenkern2012}. Respondents were randomly assigned to a question asking whether too much money is spent either on welfare (control) or assistance to the poor (treatment), with the outcome indicating whether they answered ``too much". The average treatment effect is $-37$ percentage points ($48\%$ versus $11\%$). %
Political views carry the largest share of the treatment-effect variance in CausalTreeHFD (52\%), followed by party identification (23\%) and years of education (7\%). No other covariate accounts for more than 3\%. The effect is stronger among conservative respondents, whose treatment effect is about 5 percentage points smaller than that of liberal respondents, and Party identification shows similar patterns. %
The HFD and SHAP decompositions perfectly match causal forests predictions. However, main effects and interactions are strongly entangled for SHAP.
The variable importance of \textit{grf} also ranks political views and party identification first, as displayed in Table \ref{tab:gss}, but again underestimates the influence of the main covariate, and inflates the impact of covariates with a small influence.

\section{CONCLUSION}\label{sec:conclusion}

We introduced CausalTreeHFD to interpret causal forests through the Hoeffding functional decomposition of the estimated CATE, and CausalTreeHFD is proved to be consistent.
The principle of solving the moment equation for each tree can be applied to the wider family of generalized random forests. In future work, it will thus be interesting to explore the approach for instrumental and survival forests.

\paragraph{Limitations.}
CausalTreeHFD inherits the limitations of causal forests. In particular, a limited accuracy of the forest will be reflected in the decomposition. Furthermore, reliable detection requires quite large sample sizes, exposing small datasets to spurious discoveries, since uncertainty is not provided. %

\section*{Code Availability}

CausalTreeHFD is implemented as a Python package %
and we provide the code to reproduce the experiments in the Supplementary Material. The package builds
on the \texttt{treehfd} library and calls \texttt{grf}
\citep{athey2018generalizedrandomforests} through \texttt{rpy2}. Reproducing our results requires the version of \texttt{grf} attached in the Supplementary Material, where terminal leaves are pruned to comply with the constraint for the minimum number of observations per leaf. %

\section*{AI Use Statement}

In this work, generative AI tools were used to aid with the writing and proofreading of the manuscript, as well as to develop the code extending the \texttt{grf} package with the \texttt{prune.min.sample} parameter. All methodological decisions, theoretical arguments, experimental designs were developed and verified by the authors.

\bibliographystyle{apalike}
\bibliography{references}

\clearpage
\appendix
\thispagestyle{empty}

\onecolumn
\aistatstitle{Supplementary Materials}

\section{Experiments}\label{app:exp}

\subsection{Metric Definitions}\label{app:metrics}

\paragraph{Explained variance.}
Throughout our simulated experiments, we compare algorithm accuracy using the explained variance \(R^2\), computed on a testing set, and defined as
\[
R^2 = 1 - \frac{\sum_{i=1}^n(\tau(\bX_i) - \tau_n(\bX_i))^2}{\sum_{i=1}^n(\tau(\bX_i) - \bar\tau)^2},
\]
where \(\bar\tau = \frac{1}{n} \sum_{i=1}^n \tau(\bX_i)\). In the simulations, the individual treatment effects \(\tau(\bX_i)\) are known from the data-generating process and therefore serve as a ground-truth target for evaluating the estimated CATEs. In real causal inference applications, however, individual treatment effects are unobserved, and this oracle \(R^2\) cannot be directly computed. In particular, a negative \(R^2\) in our simulations indicates that the estimated CATEs have a larger squared error relative to the ground-truth treatment effects than the corresponding sample-mean benchmark, but this comparison is not directly available in practice.

\paragraph{Hierarchical Orthogonality.}

We assess the hierarchical orthogonality conditions by evaluating the correlation coefficients \(\mathrm{Corr}(\tau^{(j)}, \tau^{(j, k)})\) for each pair of variables \((j, k)\). We compute this coefficient only for interaction components with variance at least 1\% of the output variance \(\operatorname{Var}(\tau)\) (otherwise correlation coefficients may be undefined), and report the maximum absolute value of the correlation coefficients.

\subsection{Computational Complexity of CausalTreeHFD} \label{app:complexity}

The computational complexity of learning the initial causal forest is quasilinear with respect to the sample size $n$, linear with the number of trees $B$, and sublinear with the dimension $p$: $O(B \sqrt{p} n \log^2(n))$.

Deriving the additive causal forest from the initial forest only requires to drop all observations from the estimation subsample down the trees to compute tree outputs, which has a computational complexity of $O(B n \log(n))$.

Then, CausalTreeHFD has a linear complexity with $n$ and $B$, and quadratic with $p$: $O(B n p^2)$.
However, the complexity can be linearized with $p$ by limiting the number of interaction terms extracted in each tree by TreeHFD, using the parameter \texttt{depth\_variable}.
Indeed, the most costly step in TreeHFD is solving the sparse quadratic problem for each tree, using \texttt{lsqr} function from \texttt{scipy}, which has a linear complexity with respect to the number of non-null elements in the sparse optimization matrix. For each subset of variables, the number of non-empty cells in the Cartesian tree partition is $O(n)$, and the number of main effects and second-order interactions are respectively $O(p)$ and $O(p^2)$ for fully grown trees with large $n$. The parameter \texttt{depth\_variable} in TreeHFD limits the maximum of number of interactions extracted from each tree, and makes it independent of $n$ and $p$, which leads to a linear complexity of $O(B n p)$.
Notice that building the sparse optimization matrix has the same complexity as solving the least squares problem.

\subsection{CausalTreeHFD Training Settings}\label{app:training}

\paragraph{Causal Forest Settings.}
In all simulated experiments, causal forests are trained using $B = 500$ trees and $\mathtt{min.node.size} = 5$. In addition, we implemented a leaf-pruning procedure that removes leaves containing fewer than a pre-defined number of observations used to construct leaf predictions. This parameter is called $\mathtt{prune.min.sample}$ and we fixed its default value to $20$ in our simulated experiments. It is crucial to make sure each leaf contains at least a few samples as a centering step is done at each leaf. This additional pruning step is necessary because the honesty procedure of \texttt{grf} used to construct the trees means that $\mathtt{min.node.size}$ does not guarantee a minimum number of observations in the samples used to construct each leaf prediction.

\paragraph{XGBoost Settings.} XGBoost hyperparameters were selected empirically based on preliminary experiments. The chosen hyperparameters are reported in \Cref{tab:xgboost_hyperparameters}. All other hyperparameters were left at their default XGBoost values.

\paragraph{TreeSHAP.}
TreeSHAP is also applied to the additive causal forest, as CausalTreeHFD. The additive causal forest \(\tau_n\) is re-expressed as a scikit-learn random forest whose trees carry the grf split structure and whose leaves hold the per-tree values \(\tau_{n, b}\) so that its prediction is exactly \(\tau_n(\bx) = \frac{1}{B}\sum_b \tau_{n, b}(\bx)\). The TreeSHAP explainer and our estimator therefore target an identical object and any difference between them comes from the attribution rule alone. We use \texttt{shap.TreeExplainer} with \texttt{feature\_perturbation = "tree\_path\_dependent"}, which weights subsets by the tree path coverages and thus requires no background sample.

\begin{table}%
\begin{minipage}[t]{0.58\textwidth}
\vspace{0pt}
\captionsetup{labelsep=space,singlelinecheck=false,labelfont=bf}
\hrule height.8pt depth0pt \kern2pt
\caption{Surrogate HFD (\(\tau_{\mathrm{xgb}}\))}\label{algo:xgb}
\kern2pt\hrule\kern2pt
\begin{algorithmic} [1]
\Require Covariates matrix \(\bX\), treatment assignment \(W\) and outcome \(Y\).
\State $\tau_{\mathrm{grf}} \leftarrow \mathtt{CausalForest}(\bX, W, Y)$
\State $g_n \leftarrow \mathtt{XGBoost}(\bX, \tau_{\mathrm{grf}}(\bX))$
\State $(\tau_{\mathrm{xgb}}^{(J)})_{J \in \mathcal{S}} \leftarrow \mathtt{TreeHFD}(g_n, \bX)$
\State \Return $(\tau_{\mathrm{xgb}}^{(J)})_{J \in \mathcal{S}}$
\end{algorithmic}
\kern2pt\hrule
\end{minipage}
\hfill
\begin{minipage}[t]{0.40\textwidth}
\vspace{0pt}
\centering
\captionof{table}{XGBoost hyperparameters selected.}
\label{tab:xgboost_hyperparameters}
\begin{tabular}{lc}
\toprule
\textbf{Hyperparameter} & \textbf{Selected value} \\
\midrule
Number of estimators & 500 \\
Maximum tree depth & 4 \\
Learning rate & 0.05 \\
Subsample & 0.8 \\
Column subsample & 0.8 \\
L2 regularization ($\lambda$) & 1 \\
\bottomrule
\end{tabular}
\end{minipage}
\end{table}

\paragraph{Computing Infrastructure and Software Licenses.}
All experiments were run on a single 13-inch MacBook Air (2025) with an Apple M4 chip (10 cores: 4 performance and 6 efficiency) and 16\,GB of unified memory, under macOS 26.6. No GPU was used. The experiments rely on \textit{grf} 2.6.1 \citep{athey2018generalizedrandomforests}, called from Python through \texttt{rpy2} 3.6.7, as well as \texttt{treehfd} 1.4.1 \citep{benard2025treehfd}, \texttt{XGBoost} 3.4.1 \citep{chen2016xgboost}, \texttt{shap} 0.52.0 \citep{lundberg2020local}, \texttt{scikit-learn} 1.9.0, \texttt{NumPy} 2.5.2 and \texttt{pandas} 3.0.5. These packages are distributed under the Apache 2.0 (\texttt{treehfd}, \texttt{XGBoost}), GPL-3 (\texttt{grf}), GPL-2.0-or-later (\texttt{rpy2}), MIT (\texttt{shap}) and BSD 3-Clause (\texttt{scikit-learn}, \texttt{NumPy}, \texttt{pandas}) licenses. We use all packages according to the terms of their license.

\subsection{Simulated Experiments}\label{app:exps}

We list in \Cref{tab:hfd_A} and \Cref{tab:hfd_B} the non-null HFD components of the CATE under Experiments A and B respectively.

\begin{table}%
\centering
\small
\setlength{\tabcolsep}{4pt}
\newcommand{\hfdstrut}{\vphantom{\frac{\rho}{1 + \rho^2}}}
\caption{Non-null HFD components of the CATE in Experiments A and B.}
\label{tab:hfd_AB}
\begin{subtable}[t]{0.45\textwidth}
\centering
\caption{Experiment A.}
\label{tab:hfd_A}
\begin{tabular}{c | c}
\toprule
$J$ & $\tau^{(J)}(\bX^{(J)})$ \\
\midrule
$\emptyset\hfdstrut$ & $0.25$  \\
$\{1\}\hfdstrut$ & $ 0.5 \cdot (\mathbf{1}\{X^{(1)} \geq 0\} - 0.5)$ \\
$\{2\}\hfdstrut$ & $ 0.5 \cdot (\mathbf{1}\{X^{(2)} \geq 0\} - 0.5)$  \\
$\{1, 2\}\hfdstrut$ & $(\mathbf{1}\{X^{(1)} \geq 0\} - 0.5) \cdot (\mathbf{1}\{X^{(2)} \geq 0\} - 0.5)$  \\
\bottomrule
\end{tabular}
\end{subtable}
\hfill
\begin{subtable}[t]{0.53\textwidth}
\centering
\caption{Experiment B.}
\label{tab:hfd_B}
\begin{tabular}{c | c}
\toprule
$J$ & $\tau^{(J)}(\bX^{(J)})$ \\
\midrule
$\emptyset\hfdstrut$ & $0.5 \cdot \rho$  \\
$\{1\}\hfdstrut$ & $ 0.5 \cdot \frac{\rho}{1 + \rho^2}({X^{(1)}}^2 - 1)$ \\
$\{2\}\hfdstrut$ & $0.5 \cdot \frac{\rho}{1 + \rho^2}({X^{(2)}}^2 - 1)$  \\
$\{1, 2\}\hfdstrut$ & $0.5 \cdot (\rho \frac{1 - \rho^2}{1 + \rho^2} - \frac{\rho}{1 + \rho^2}({X^{(1)}}^2 + {X^{(2)}}^2) + {X^{(1)}}{X^{(2)}})$  \\
\bottomrule
\end{tabular}
\end{subtable}
\end{table}

Notice that for Experiment B, if the covariates were independent ($\rho = 0$), the decomposition would simplify to a single non-null interaction component $\tau^{(\{1,2\})}(X^{(1)}, X^{(2)}) = 0.5 \cdot X^{(1)} X^{(2)}$.

Additionally, we provide results in \Cref{fig:expAB} of our experiments for different sample sizes (\(n\) = 500, 1000, 2000 and 5000) to show that the decomposition accuracy increases as we add more data.

\begin{figure}%
    \centering
    \begin{subfigure}[b]{0.49\linewidth}
        \centering
        \includegraphics[width=\linewidth]{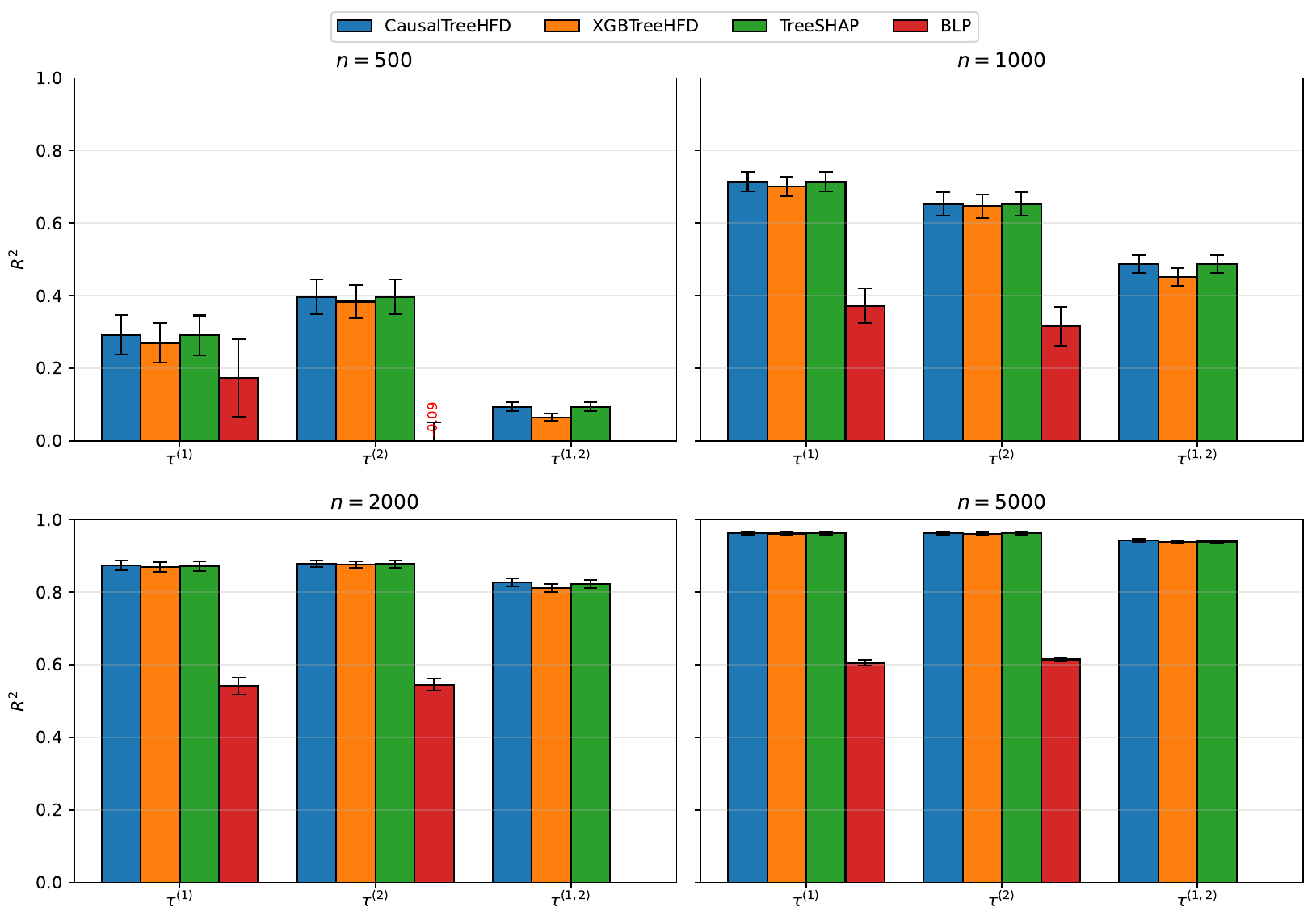}
        \caption{Experiment A.}
        \label{fig:expA}
    \end{subfigure}
    \hfill
    \begin{subfigure}[b]{0.49\linewidth}
        \centering
        \includegraphics[width=\linewidth]{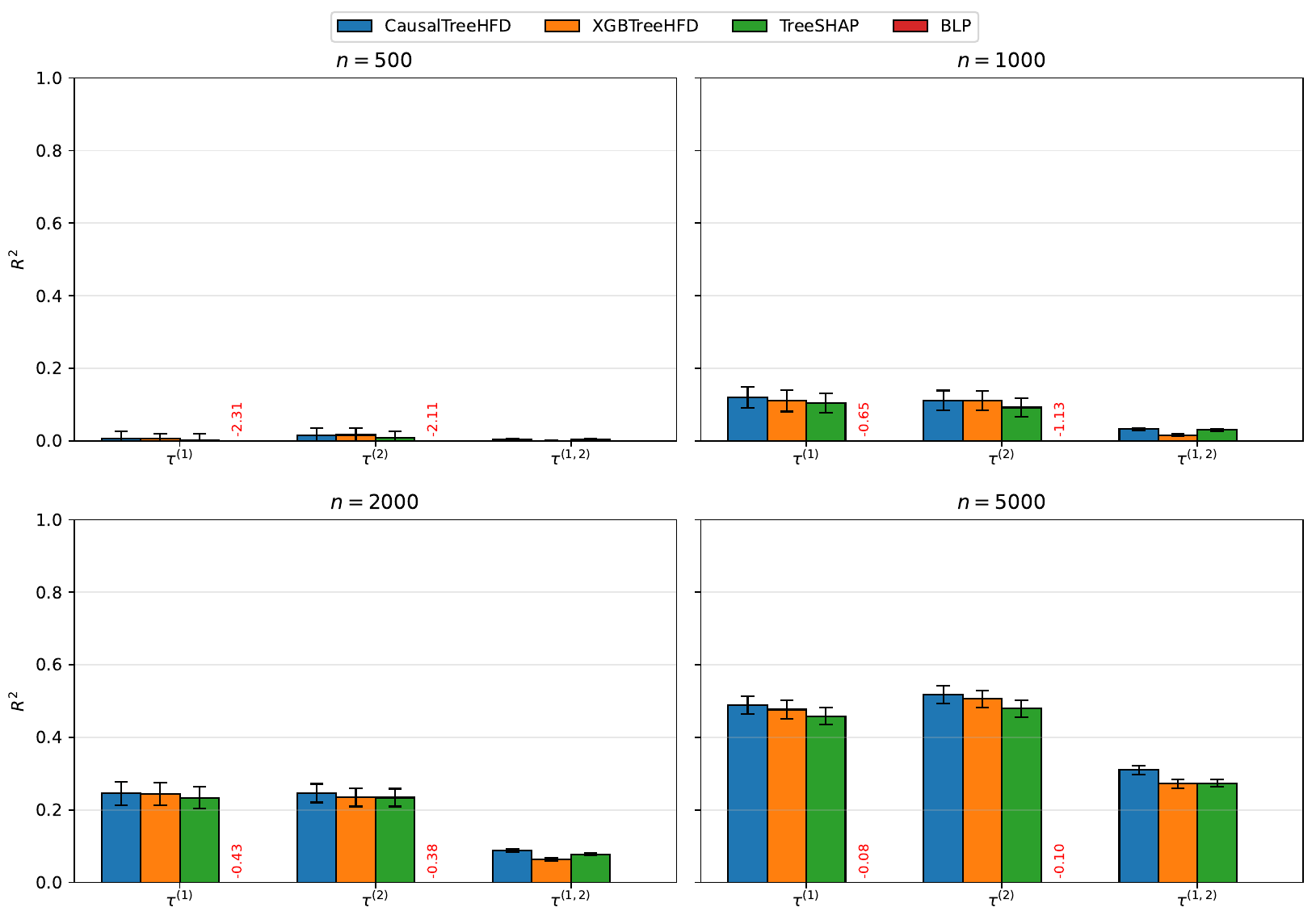}
        \caption{Experiment B.}
        \label{fig:expB}
    \end{subfigure}
    \caption{Explained variance and standard errors for each non-null HFD component across sample sizes in Experiments A and B.}
    \label{fig:expAB}
\end{figure}

\Cref{fig:plots_expA} and \Cref{fig:plots_expB} plot the components of the decomposition for one run of Experiments A and B alongside the true decomposition, and show that the three methods correctly estimate the expected decomposition. Notice that interaction terms are more noisy for SHAP. In \Cref{fig:plots_expB}, the inputs have a Gaussian distribution, and consequently, there are very few data points with coordinates smaller than $-2$ and above $2$, and trees are constant in these areas by construction.

\begin{figure}%
    \centering
    \begin{subfigure}[b]{0.8\linewidth}
        \centering
        \includegraphics[width=\linewidth]{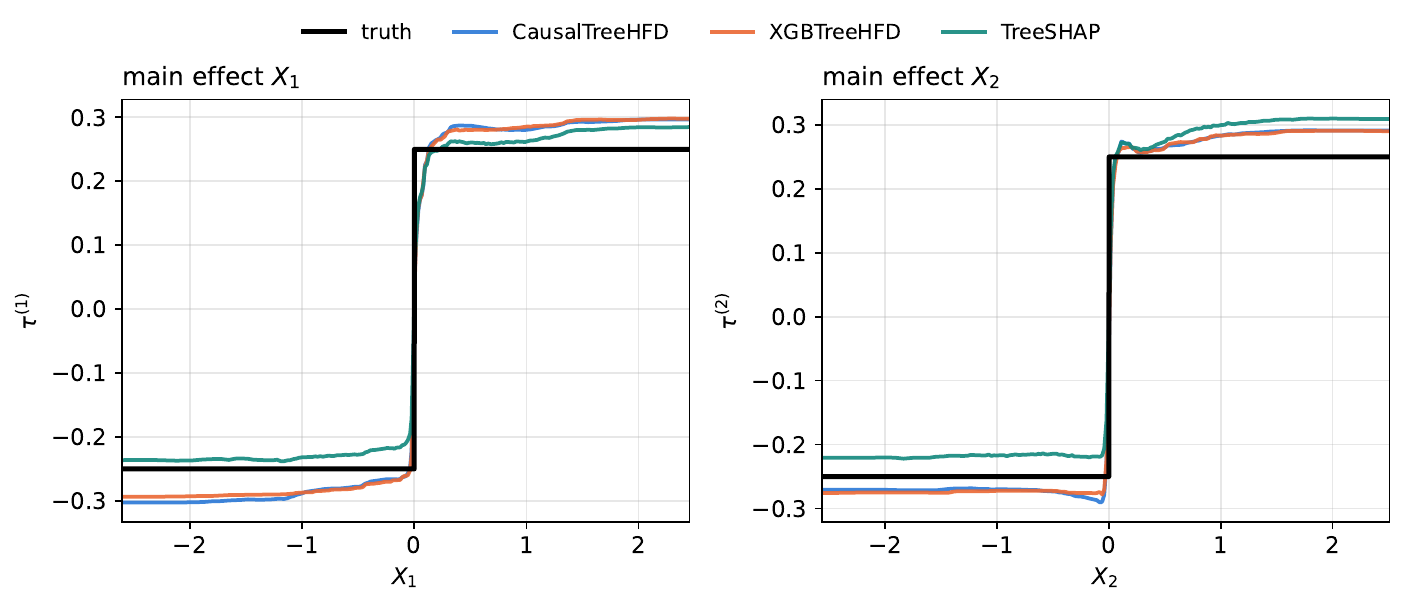}
        \caption{Main effects.}
    \end{subfigure}

    \vspace{0.5em}

    \begin{subfigure}[b]{0.8\linewidth}
        \centering
        \includegraphics[width=\linewidth]{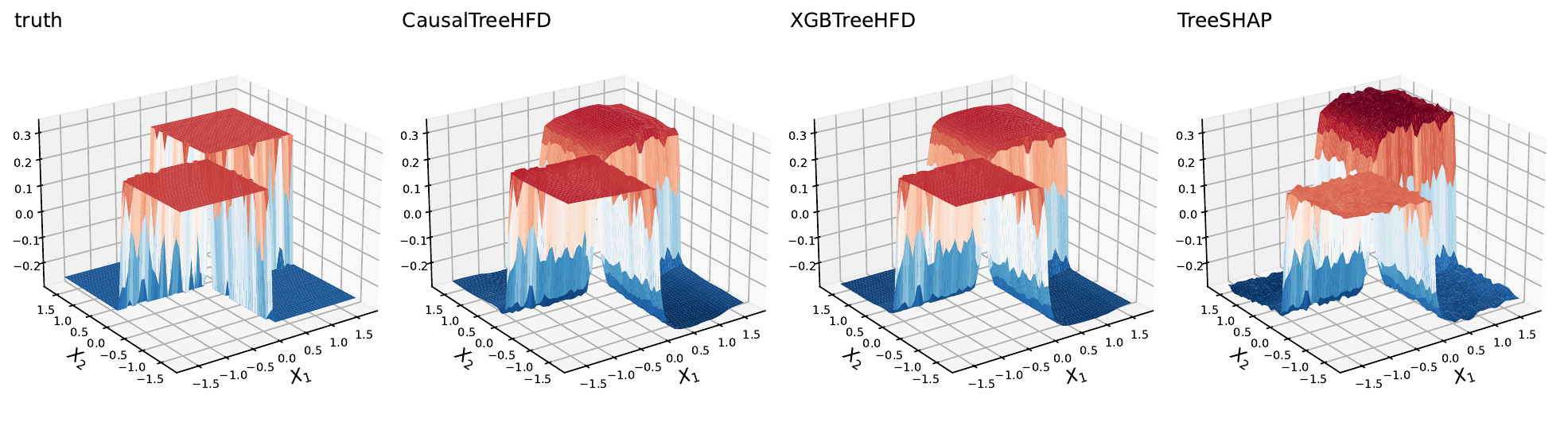}
        \caption{Interaction effect.}
    \end{subfigure}

    \caption{Main effects and interaction effects of variables \(\bX^{(1)}\) and \(\bX^{(2)}\) for Experiment A.}
    \label{fig:plots_expA}
\end{figure}

\begin{figure}%
    \centering
    \begin{subfigure}[b]{0.8\linewidth}
        \centering
        \includegraphics[width=\linewidth]{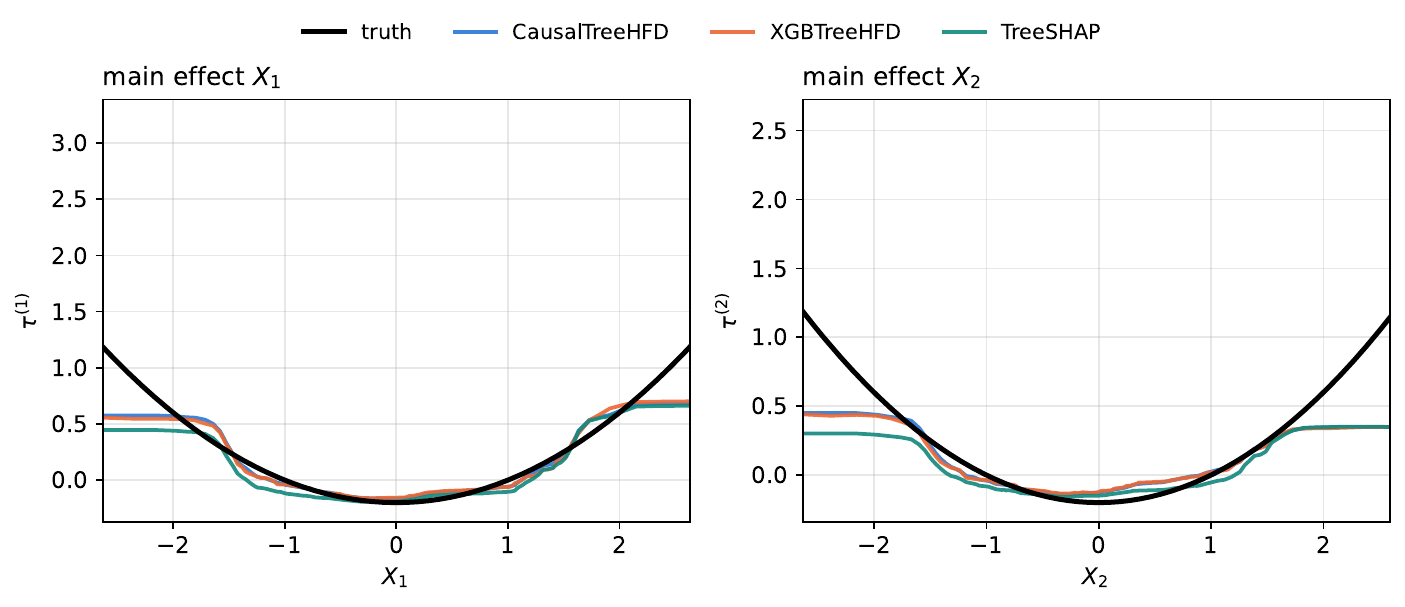}
        \caption{Main effects.}
    \end{subfigure}

    \vspace{0.5em}

    \begin{subfigure}[b]{0.8\linewidth}
        \centering
        \includegraphics[width=\linewidth]{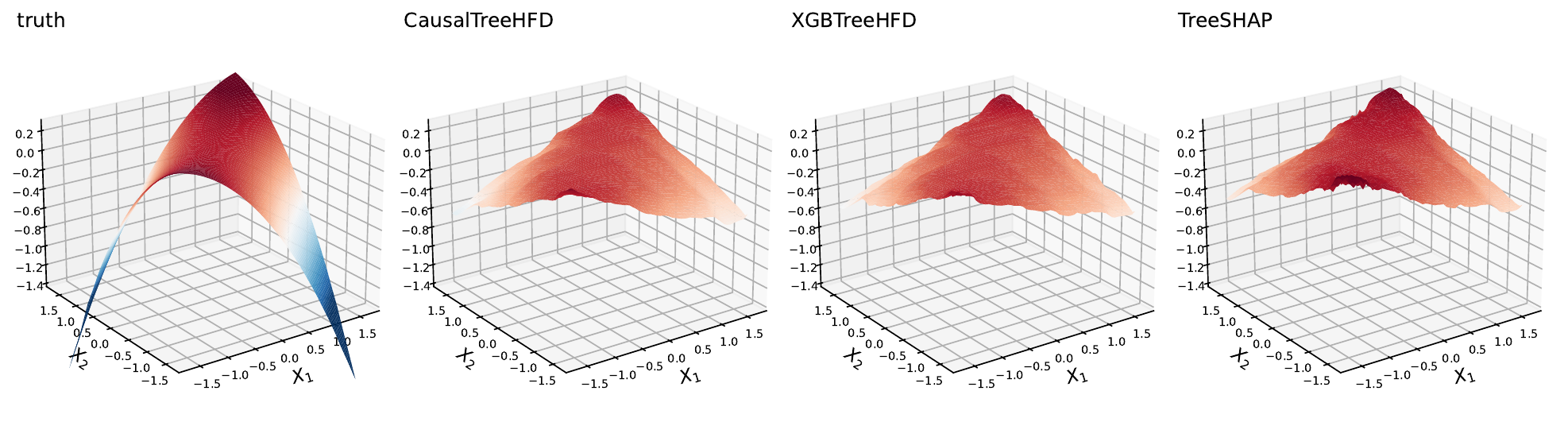}
        \caption{Interaction effect.}
    \end{subfigure}

    \caption{Main effects and interaction effects of variables \(\bX^{(1)}\) and \(\bX^{(2)}\) for Experiment B.}
    \label{fig:plots_expB}
\end{figure}

\subsection{Real Data Experiments}

\paragraph{NHEFS.}

The NHEFS dataset tracks the weight of individuals measured once in 1971 and again in 1982 \citep{madans1986}. As is done in Chapter 12 of \cite{hernanrobins2020}, treatment is smoking cessation between the two visits and the outcome is the weight change in kilograms. Of the 1629 individuals in the dataset, 63 have no recorded weight in 1982, which leaves \(n = 1566\) complete cases, of whom 25.7\% quit smoking. The average treatment effect as estimated by a causal forest is 3.30 kg.

For our analysis, we keep every baseline covariate recorded for all participants, leaving \(p=41\) covariates. We train a causal forest using \(B = 2000\), \(\mathtt{min.node.size} = 5\), and \(\mathtt{prune.min.sample} = 16\) (approximately 1\% of the number of samples). We run CausalTreeHFD on the resulting forest and analyze its components. \Cref{tab:nhefs} reports, for each covariate, the variance of its first-order component divided by the sum of the variances of all the components of the decomposition. This value can be understood as a variable importance metric, which we compare to the split-based variable importance implemented in the \textit{grf} package. \Cref{fig:plots_nhefs} shows the main components of the estimated decomposition.

\begin{figure*}%
    \centering
    \includegraphics[width=\linewidth]{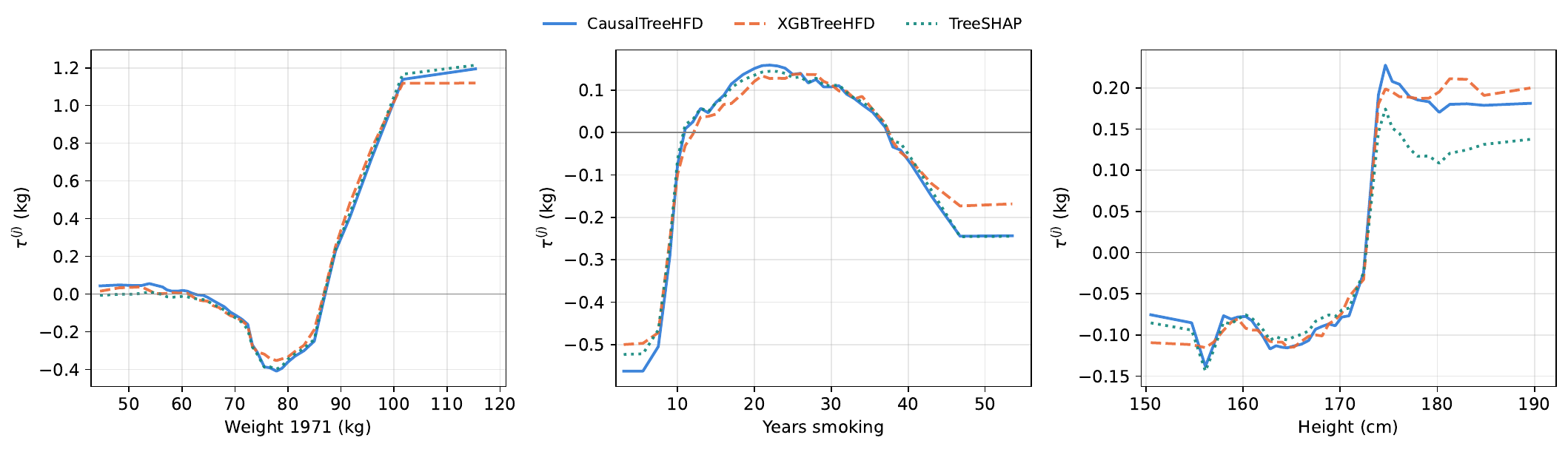}
    \caption{Main effects of the three most important variables in the NHEFS dataset.}
    \label{fig:plots_nhefs}
\end{figure*}

\paragraph{GSS Welfare.}

The Welfare dataset comes from a General Social Survey experiment \citep{greenkern2012}. Respondents were asked whether too much money is spent on welfare (control) or on assistance to the poor (treatment). The wording was randomly assigned, and the outcome is whether the respondent answered ``too much''. We keep $p = 31$ covariates and drop every respondent with a missing value, which leaves \(n = 13198\) of the 36501 respondents, of whom 51.4\% received the treatment wording. The average treatment effect as estimated by a causal forest is $-36.6$ percentage points.

We train a causal forest using \(B = 2000\), \(\mathtt{min.node.size} = 5\), and \(\mathtt{prune.min.sample} = 130\) (approximately 1\% of the number of samples). We run CausalTreeHFD on the resulting forest and compute the same variable importance metric as for NHEFS. \Cref{tab:gss} compares it to the split-based variable importance implemented in the \textit{grf} package. \Cref{fig:plots_gss} shows the main components of the 3 most important variables of the estimated decomposition.

Hierarchical orthogonality requires an interaction to carry no information already held by the first-order terms it contains, that is \(\E[\tau^{(\mathrm{partyid},\,\mathrm{polviews})} \mid \bX^{(j)}] = 0\) for \(j \in \{\mathrm{partyid}, \mathrm{polviews}\}\). \Cref{fig:orth_gss} plots these two conditional means against the covariate being conditioned on. Both HFD estimators stay flat and close to zero, whereas the TreeSHAP curve increases broadly across the political spectrum. It means that SHAP interaction component captures a signal that should have been included in first-order components. This is quite expected, since SHAP is not designed to enforce hierarchical orthogonality. However, this property improves interpretability of decompositions, because it ensures that interactions and main effects are separated, and only includes an interaction if required, which favors high sparsity.

\begin{figure*}%
    \centering
    \includegraphics[width=\linewidth]{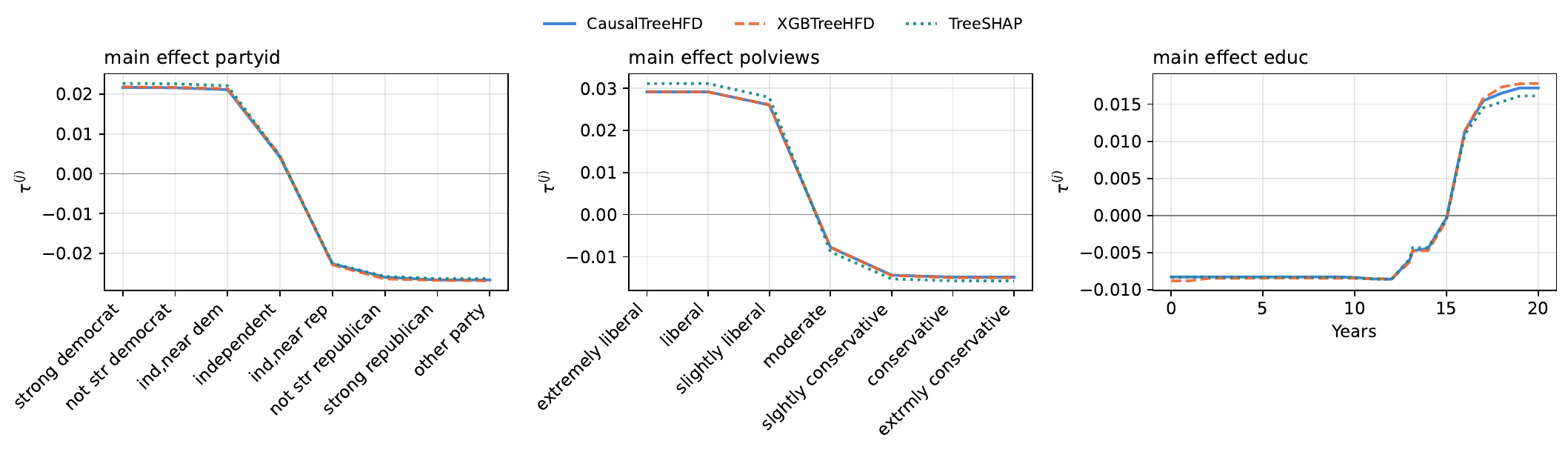}
    \caption{Main effects of the three most important variables in the GSS Welfare dataset.}
    \label{fig:plots_gss}
\end{figure*}

\begin{table}%
\centering
\setlength{\tabcolsep}{4pt}
\begin{minipage}[t]{0.48\linewidth}
\centering
\caption{NHEFS: share of the total component variance carried by each covariate under CausalTreeHFD next to the split-based importance of \textit{grf}.}
\label{tab:nhefs}
\begin{tabular}{l cc}
\toprule
\textbf{Covariate} & \textbf{CausalTreeHFD} & \textbf{\textit{grf}} \\
\midrule
Weight 1971       & 0.51 & 0.26 \\
Years smoking     & 0.12 & 0.13 \\
Height            & 0.10 & 0.10 \\
Smoking intensity & 0.08 & 0.07 \\
Age               & 0.07 & 0.10 \\
\bottomrule
\end{tabular}
\end{minipage}\hfill
\begin{minipage}[t]{0.48\linewidth}
\centering
\caption{GSS Welfare: share of the total component variance carried by each covariate under CausalTreeHFD next to the split-based importance of \textit{grf}.}
\label{tab:gss}
\begin{tabular}{l cc}
\toprule
\textbf{Covariate} & \textbf{CausalTreeHFD} & \textbf{\textit{grf}} \\
\midrule
Political views        & 0.52 & 0.30 \\
Party identification   & 0.23 & 0.18 \\
Years of education     & 0.07 & 0.08 \\
Hours worked last week & 0.03 & 0.07 \\
Industry               & 0.03 & 0.08 \\
\bottomrule
\end{tabular}
\end{minipage}
\end{table}

\begin{figure*}%
    \centering
    \includegraphics[width=\linewidth]{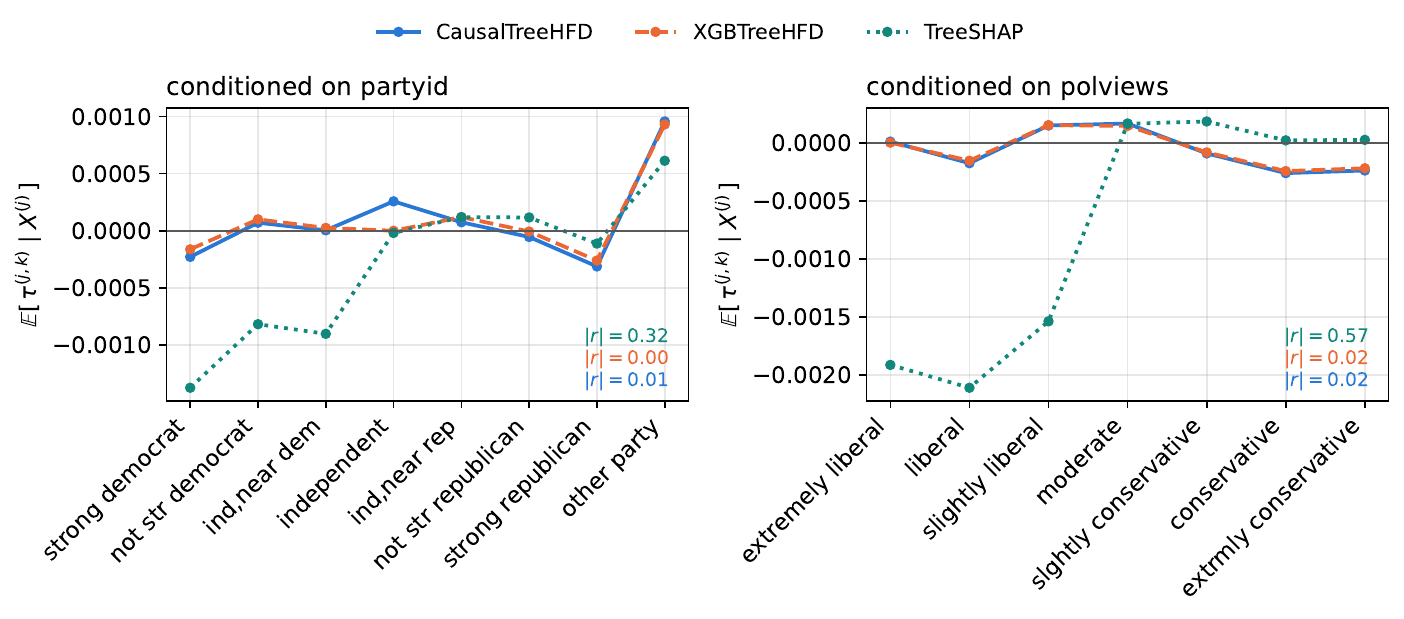}
    \caption{Hierarchical orthogonality on GSS Welfare for the \texttt{partyid} \(\times\) \texttt{polviews} interaction. Each panel shows \(\E[\tau^{(\mathrm{partyid},\,\mathrm{polviews})} \mid \bX^{(j)}]\) for \(j \in \{\mathrm{partyid}, \mathrm{polviews}\}\), which a valid decomposition sets to zero.}
    \label{fig:orth_gss}
\end{figure*}

\section{Proofs of theorems}\label{app:proofs}

\subsection{Proof of Theorem \ref{th:one}}

\begin{proof}[Proof of Theorem \ref{th:one}]

The proof relies on deconstructing the output of the estimate as a continuous function of several intermediate estimates, which we will show converge to their oracle counterparts.

Notice that Theorem \ref{th:one} considers the causal forest estimate  of Equation (\ref{eq:taunb}) without the local centering of $W$ and $Y$ with regression forests, following the theoretical analysis of \citet{athey2018generalizedrandomforests}, which also mentions in Section 6.1.1. that consistency still holds with any local centering step performed with independent data.

We first establish the consistency of a single tree estimate $\tau_{n,b}(\bx)$. Since
\[
\tau_n(\bx) = \frac{1}{B} \sum_{b=1}^B \tau_{n,b}(\bx),
\]
the consistency of the ensemble follows immediately from the consistency of each tree, since $B$ is finite. As written in Equation (\ref{eq:taunb}), each tree estimate is defined by
\begin{align*}
    \tau_{n,b}(\bx) &= \frac{\sum\limits_{i=1}^n \alpha_{ib}(\bx) (W_i - \sum\limits_{i=1}^n \alpha_{ib}(\bx) W_i)(Y_i - \sum\limits_{i=1}^n \alpha_{ib}(\bx) Y_i)}{\sum\limits_{i=1}^n \alpha_{ib}(\bx) (W_i - \sum\limits_{i=1}^n \alpha_{ib}(\bx) W_i)^2} \\
    &= \frac{\sum\limits_{i=1}^n \alpha_{ib}(\bx) W_i Y_i - (\sum\limits_{i=1}^n \alpha_{ib}(\bx) W_i) (\sum\limits_{i=1}^n \alpha_{ib}(\bx) Y_i)}{\sum\limits_{i=1}^n \alpha_{ib}(\bx) W_i^2 - (\sum\limits_{i=1}^n \alpha_{ib}(\bx) W_i)^2}.
\end{align*}

The estimate is thus a continuous function of several leaf-wise averages. It is therefore sufficient to show that each of these averages converges in probability to its oracle counterpart.
Hence, we begin with
\[
\mu_{n, b}(\bx) = \sum_{i=1}^n \alpha_{ib}(\bx) Y_i,
\]
and show that
\[
\mu_{n, b}(\bx) \xrightarrow{p} \mu(\bx), \qquad \mu(\bx) = \mathbb{E}[Y \mid \bX=\bx].
\]
The same arguments apply to the remaining moments appearing in the numerator and denominator of $\tau_{n, b}(\bx)$.

\vspace{1em}

Recall that $n_b(\bx)$ denotes the number of observations of the estimation subsample contained in the leaf $L_b(\bx)$, and let $\mathcal{T}_b = (\mathcal{I}_b, \Theta_b)$ be the tree structure, where \(\mathcal{I}_b\) is the subsample used to determine the recursive partition and \(\Theta_b\) collects any additional randomization introduced during tree construction (e.g. random feature selection). By honesty of the tree construction, the weight $\alpha_{ib}(\bx)$ is independent of the outcome $Y_i$ used for estimation of the cell outputs, conditional on the input $\bX_i$.

Notice that Specification \ref{spec:one} enforces that the minimum number of points in each terminal leaf is at least $k_n$. In practice, trees are grown until terminal leaves cannot be further split without violating this constraint. The cell size is measured using the subsample $\mathcal{I}_b$ used for the partition construction. However, because of the honesty property, cell outputs are computed using an independent subsample $\mathcal{I}'_b$ of the same size as $\mathcal{I}_b$, and only observations from $\mathcal{I}'_b$ can belong to $L_b(\bx)$. Consequently, some terminal leaves may contain a smaller number of points than $k_n$, and some splits at the bottom of the tree must be pruned to comply with Specification \ref{spec:one}.
This pruning mechanism implies that the subsample $\mathcal{I}'_b$ is also involved to build the tree partition. However, only the inputs $\bX_i$ are used to measure cell sizes, but not the outputs $Y_i, W_i$. Therefore, the honesty property is still valid.

\paragraph{Bias.}
We first condition on the tree structure to write
\begin{align*}
    \left| \mathbb{E}[\mu_{n, b}(\bx)] - \mu(\bx) \right| \vphantom{\sum_{i}} &= \left| \mathbb{E}[\mathbb{E}[\mu_{n, b}(\bx) \mid \mathcal{T}_b, \bX_1, \hdots, \bX_n] - \mu(\bx)] \right| \vphantom{\sum_{i}} \\
    &= \left| \mathbb{E}\left[\frac{1}{n_b(\bx)} \sum_{i=1}^n \left(\mathbb{E}[Y_i \mid \bX_i] - \mu(\bx)\right) \mathds{1}_{\bX_i \in L_b(\bx)} \right]  \right| \\
    &\leq \mathbb{E}\left[\frac{1}{n_b(\bx)} \sum_{i=1}^n \left| \mathbb{E}[Y_i \mid \bX_i] - \mu(\bx) \right| \mathds{1}_{\bX_i \in L_b(\bx)} \right].
\end{align*}

Since every observation in the leaf lies at distance at most \(\text{diam}(L_b(\bx))\) from \(\bx\), and using the Lipschitz continuity of \(\mu\) with the constant denoted by $C_{\mu}$,
\[
\left| \mathbb{E}[Y_i \mid \bX_i] - \mu(\bx) \right| \mathds{1}_{\bX_i \in L_b(\bx)} \leq C_{\mu}\, \text{diam}(L_b(\bx)).
\]
Therefore,
\[
\left| \mathbb{E}[\mu_{n, b}(\bx)] - \mu(\bx) \right| \leq C_{\mu}\, \mathbb{E}[\text{diam}(L_b(\bx))].
\]
Then, we apply Lemma \ref{lemma:diam} to get that, in probability,
\[
\mathrm{diam}(L_b(\bx)) \rightarrow 0,
\]
and since cell diameters are bounded, $\E[\mathrm{diam}(L_b(\bx))] \rightarrow 0$, and we obtain
\[
\left| \mathbb{E}[\mu_{n, b}(\bx)] - \mu(\bx) \right| \rightarrow 0.
\]

\paragraph{Variance.}

Using the law of total variance,
\begin{align*}
     \operatorname{Var}(\mu_{n, b}(\bx)) =  &\operatorname{Var}(\mathbb{E}[\mu_{n, b}(\bx) \mid \mathcal{T}_b, \bX_1, \hdots, \bX_n]) + \mathbb{E}[ \operatorname{Var}(\mu_{n, b}(\bx) \mid \mathcal{T}_b, \bX_1, \hdots, \bX_n)].
\end{align*}

We study each term separately.

\paragraph{First term.}

Because
\begin{align*}
\mathbb{E}[\mu_{n, b}(\bx) \mid \mathcal{T}_b,& \bX_1, \hdots, \bX_n] = \frac{1}{n_b(\bx)} \sum_{i=1}^n \mathbb{E}[Y_i \mid \bX_i] \mathds{1}_{\bX_i \in L_b(\bx)},
\end{align*}
we obtain
\begin{align*}
    \operatorname{Var}(\E[\mu_{n, b}(\bx) \mid \mathcal{T}_b, \bX_1, \hdots, \bX_n]) 
    &= \min_c \mathbb{E}\left[(\E[\mu_{n, b}(\bx) \mid \mathcal{T}_b, \bX_1, \hdots, \bX_n] - c)^2\right] \\ 
    &\leq \mathbb{E}\left[(\E[\mu_{n, b}(\bx) \mid \mathcal{T}_b, \bX_1, \hdots, \bX_n] - \mu(\bx))^2\right] \\
    &\leq \mathbb{E}\Big[ \Big(\frac{1}{n_b(\bx)} \sum_{i=1}^n (\mathbb{E}[Y_i \mid \bX_i]  - \mu(\bx)) \mathds{1}_{\bX_i \in L_b(\bx)} \Big)^2\Big] \\
    &\leq \mathbb{E}\Big[ \Big(\frac{C_{\mu}}{n_b(\bx)} \sum_{i=1}^n \|\bX_i - \bx \|_2 \mathds{1}_{\bX_i \in L_b(\bx)} \Big)^2\Big] \\
    &\leq C_{\mu}^2\, \mathbb{E}[\text{diam}(L_b(\bx))^2].
\end{align*}
Since the leaf diameters converge to zero,
\[
 \operatorname{Var}(\mathbb{E}[\mu_{n, b}(\bx) \mid \mathcal{T}_b, \bX_1, \hdots, \bX_n]) \rightarrow 0.
\]

\paragraph{Second term.}

Since
\begin{align*}
\mu_{n, b}(\bx) = \frac{1}{n_b(\bx)} \sum_{i=1}^n Y_i \mathds{1}_{\bX_i \in L_b(\bx)},
\end{align*}
and the outputs $Y_i$ are independent conditional on $\bX_1, \hdots, \bX_n$ and $\mathcal{T}_b$.

From Assumption \ref{asm:two}, conditional second moments are uniformly bounded, that is
\[
\mathbb{E}[Y^2 \mid \bX=\bx] \le C.
\]
Then, for all \(\bx \in \mathcal{X}\), we have
\begin{align*}
    \operatorname{Var}(\mu_{n, b}(\bx)&\mid \mathcal T_b, \bX_1, \hdots, \bX_n) \\
    &= \frac{1}{n_b(\bx)^2}
       \sum_{i=1}^n
       \operatorname{Var}(Y_i \mid \bX_i) \mathds{1}_{\bX_i \in L_b(\bx)} \\
    &\leq \frac{1}{n_b(\bx)^2}
       \sum_{i=1}^n
       \E[Y_i^2 \mid \bX_i] \mathds{1}_{\bX_i \in L_b(\bx)} \\
    &\leq \frac{C}{n_b(\bx)}
     \leq \frac{C}{k_n}.
\end{align*}
Taking expectation yields
\[
\mathbb{E}[ \operatorname{Var}(\mu_{n, b}(\bx) \mid \mathcal{T}_b, \bX_1, \hdots, \bX_n)] \leq \frac{C}{k_n},
\]
where \(k_n\) denotes the minimum leaf size. Since \(k_n \rightarrow \infty\),
\[
\mathbb{E}[ \operatorname{Var}(\mu_{n, b}(\bx) \mid \mathcal{T}_b, \bX_1, \hdots, \bX_n)] \rightarrow 0.
\]
Combining both terms,
\[
 \operatorname{Var}(\mu_{n, b}(\bx)) \rightarrow 0.
\]
Since both the bias and variance vanish, we conclude that
\[
\mu_{n, b}(\bx) \xrightarrow{p} \mu(\bx).
\]

The same argument applies to the estimates of \(\mathbb{E}[WY \mid \bX=\bx]\), \(\mathbb{E}[W \mid \bX=\bx]\) and \(\mathbb{E}[W^2 \mid \bX=\bx]\). By the Continuous Mapping Theorem, it follows that
\[
\tau_{n, b}(\bx) \xrightarrow{p} \tau(\bx),
\]
where the denominator is not null from the positivity property of Assumption \ref{asm:causal}. 
Consistency of the forest estimate \(\tau_n(\bx)\) then follows immediately from averaging over all trees \(b \in \{1, ..., B\} \), since $B$ is finite.

\end{proof}

\begin{lemma} \label{lemma:diam}
    If Assumptions~\ref{asm:causal}--\ref{asm:three} and \Cref{spec:one} are satisfied, then for $\bx \in \mathcal{X}$ and any tree $b \in \{1, \hdots, B\}$, we have in probability
    \begin{align*}
        \mathrm{diam}(L_b(\bx)) \rightarrow 0.
    \end{align*}
\end{lemma}

\begin{proof}[Proof of Lemma \ref{lemma:diam}]
    According to Specification \ref{spec:one}, the minimal leaf size is $k_n$. By construction, tree cells are split until the leaf size constraint cannot be satisfied with an additional split. Consequently, the number of points $N_{n, b}(\bx)$ in the final leaf where $\bx \in \mathcal{X}$ falls in the $b$-th tree, satisfies
    \begin{align*}
        N_{n, b}(\bx) < 2 k_n.
    \end{align*}
    Otherwise, it is possible to further split the terminal leaf.

    From Specification \ref{spec:one}, tree splits are balanced, which means that at least a fraction $\gamma > 0$ of the observations of the parent node goes in each child node. If $c_n(\bx)$ is the number of splits made to reach the terminal leaf containing $\bx$, we have
    \begin{align*}
        \frac{s_n}{2} \gamma^{c_n(\bx)} \leq N_{n, b}(\bx).
    \end{align*}

    Combining the last two inequalities, we get
    \begin{align*}
        c_n(\bx) > \frac{\log(s_n / (4 k_n))}{\log(1 / \gamma)}.
    \end{align*}
    According to Specification \ref{spec:one}, $s_n / k_n \to \infty$, and thus, $c_n(\bx) \to \infty$ in probability. Additionally, the probability to split on any input variable is strictly positive, which means that the number of splits on each variable to reach the cell $L_b(\bx)$ also grows to infinity.
    Notice that this result still holds when the pruning step described above is applied. Indeed, the terminal leaves containing $\bx$ in the pruned tree still have a depth growing to infinity.

    Given this growing number of splits on each direction, and that splits are constrained to be away from the edge of cells (balanced splits), it follows that $\mathrm{diam}(L_b(\bx)) \rightarrow 0$ in probability, as shown in the proof of Lemma $2$ from \citet{meinshausen2006quantile}.
    
\end{proof}

\subsection{Proofs of Theorems \ref{th:two} and \ref{th:three}}

\begin{proof}[Proof of Theorem \ref{th:two}.]
    Since $\tau_n$ is the additive causal forest, which takes the form of an additive tree ensemble as regression forests, we can directly apply Theorem $2$ from TreeHFD article to obtain the result. Notice that the constant $1/B$ is not absorbed in the tree definition here.
\end{proof}

\begin{proof}[Proof of Theorem \ref{th:three}.]
    We first apply Theorem \ref{th:one} to state the convergence in probability of the additive causal forest $\tau_n$. According to Assumption \ref{asm:one}, $\tau$ is a Lipschitz function (as the difference of Lipschitz functions), and since $\tau$ is defined over the hyperrectangle $\mathcal{X}$, which is a compact set, $\tau$ is bounded. Additionally, the trees $\tau_{n,b}$ are bounded by assumption, and therefore the $\mathrm{L}^2$-consistency of the additive causal forest holds, that is
    \begin{align*}
        \mathbb{E}[(\tau_n(\bX) - \tau(\bX))^2] \xrightarrow[n \to \infty]{} 0.
    \end{align*}
    Then, all assumptions of Corollary $1$ of TreeHFD article are satisfied, and we can directly apply the corollary to get the final result.
\end{proof}

\subsection{Proofs of Theorem \ref{thm:causaltreehfd_bound} and Corollary \ref{cor:convergence}} \label{app:proof_thm5}

We define $f_{n,b}$ as the empirical density of the inputs computed over the Cartesian tree partition $\mathcal{A}_b$ of tree $b$ with $\Dn$, and $\smash{f_{n,b}^{(J)}}$ is the marginal density for variables in $J \in \mathcal{S}$. 

\begin{proof}[Proof of Theorem \ref{thm:causaltreehfd_bound}]
    Theorem \ref{thm:causaltreehfd_bound} extends Theorem $6$ from TreeHFD article \citep{benard2025treehfd}, where the bound is now applied to the empirical estimates output by CausalTreeHFD $\{ \tau_n^{(J)} \}_{J \in \mathcal{S}}$, instead of the theoretical HFD adapted to tree ensembles $\{\tau_n^{\star (J)}\}_{J \in \mathcal{S}}$ of the fitted causal forest.
    
    Hence, the core of the proof is based on Lemma \ref{lemma:emp_hfd} to state the equivalence between the minimization of $\mathcal{R}_{n,b}$ performed by CausalTreeHFD and the HFD of $\tau_{n,b}$ with respect to the truncated empirical densities $f_{n,b}$, computed over the Cartesian tree partitions.
    Then, the proof follows TreeHFD article, using the extended Lemmas \ref{lemma:HFD_equivalence} and \ref{lemma:bound_tree_HFD} below.
\end{proof}

\begin{lemma} \label{lemma:emp_hfd}
    Let Assumption \ref{asm:three} and Specification \ref{spec:causaltreehfd} be satisfied.
    The CausalTreeHFD estimate $\{\tau_{n,b}^{(J)}\}_{J \in \mathcal{S}}$, defined as the minimizer of $\mathcal{R}_{n,b}$ for each tree, is the HFD of $\tau_{n,b}$ with respect to the truncated empirical density $f_{n,b}$, and therefore satisfies for all trees $b \in \{1, \hdots, B\}$,
    \begin{align*}
        \tau_{n,b}(\bX) = \sum_{J \in \mathcal{S}} \tau_{n,b}^{(J)}(\bX^{(J)}),
    \end{align*}
    and for all $J \in \mathcal{S}$, all $j \in J$, and all $A \in \mathcal{A}_b^{(J \setminus j)}$,
    \begin{align} \label{eq:emp_ortho}
        \sum_{B \in \mathcal{A}_b^{(J)}: B^{(J \setminus j)} = A} \tau_{n,b}^{(J)}(B) p_n(B) = 0,
    \end{align}
    where $\tau_{n,b}^{(J)}(B)$ is the value of $\tau_{n,b}^{(J)}$ in cell $B$, and $p_n(B)$ is the empirical probability that an input belongs to $B$, given by $p_n(B) = f_{n,b}^{(J)}(B) \mathrm{Vol}(B)$.
\end{lemma}

\begin{proof}[Proof of Lemma \ref{lemma:emp_hfd}]
    According to Specification \ref{spec:causaltreehfd}, the empirical densities $f_{n,b}$ are truncated from above and below. Therefore, we can apply Theorem $1$ from \citet{chastaing2012hoeffding} to state the HFD $\smash{\{\tau_{n,b}^{(J)}\}_{J \in \mathcal{S}}}$ of each tree $\tau_{n,b}$ with respect to the empirical density $f_{n,b}$, which is unique and satisfies the empirical orthogonality conditions of Equation (\ref{eq:emp_ortho}). Notice that because trees and empirical densities are constant on each cell of the Cartesian tree partition, it is equivalent to state the orthogonality conditions conditional on each cell, as shown in the proof of Theorem $2$ of TreeHFD article.
    
    By definition of TreeHFD algorithm, the loss $\mathcal{R}_{n,b}$ is defined for each tree $b$ and a set of parametrized functions $\smash{\{\mu_{b}^{(J)}\}}$ (piecewise constant on the Cartesian tree partition) by
    \begin{align*}
        \mathcal{R}_{n,b} (\{ \mu_{b}^{(J)} \}_{J \in \mathcal{S}})
        = &\sum_{A \in \mathcal{A}_b} p_n(A) \big(\tau_{n, b}(A) - \sum_{J \in \mathcal{S}} \mu_{b}^{(J)}(A^{(J)})\big)^2
          \\ &+ \sum_{J \in \mathcal{S}} \sum_{j \in J}  \sum_{A \in \mathcal{A}_b^{(J \setminus j)}} 
          p_n(A)^{-1} \times \Big(\sum_{B \in \mathcal{A}_b^{(J)}: B^{(J \setminus j)} = A} \mu_{b}^{(J)}(B) p_n(B)\Big)^2.
    \end{align*}
    Notice that according to Specification \ref{spec:causaltreehfd}, the loss of CausalTreeHFD $\mathcal{R}_{n,b}$ is built with all variable subsets $J \in \mathcal{S}$ and with the truncated empirical density $f_{n,b}$ of tree $b$. This means that empty cells are also included in $\mathcal{R}_{n,b}$ with weights defined from the truncating constants of the density and cell volumes.
    Then, it is clear that 
    \begin{align*}
        \mathcal{R}_{n,b}(\{\tau_{n,b}^{(J)}\}_{J \in \mathcal{S}}) = 0.
    \end{align*} 
    When the orthogonality constraints are violated or the decomposition does not add up to $\tau_{n,b}$, we have $\mathcal{R}_{n,b} > 0$, since Lemma 4.1 from \citet{Hooker01092007} gives the equivalence of the orthogonality conditions for $j \in J$ and $I \subsetneq J$. Consequently, $\smash{\{\tau_{n,b}^{(J)}\}_{J \in \mathcal{S}}}$ is the unique minimizer of $\mathcal{R}_{n,b}$, output by CausalTreeHFD.
\end{proof}

\begin{lemma} \label{lemma:HFD_equivalence}
    If Assumption \ref{asm:three} and Specification \ref{spec:causaltreehfd} are satisfied, and $\{\mu_{n,b}^{(J)}\}_J$ is the HFD of $\tau_{n,b}(\bX)$ with respect to the theoretical density $f$ of $\bX$, then we have for $J \in \mathcal{S}$,
    \begin{align*}
        \tau_{n,b}^{(J)} = \mu_{n,b}^{(J)} \frac{f^{(J)}}{f_{n,b}^{(J)}} - \delta_{n,b}^{(J)},
    \end{align*}
    where $\{\delta_{n,b}^{(J)}\}_J$ is the HFD of $\delta_{n,b}$ with respect to the empirical density $f_{n,b}$, with
    \begin{align*}
        \delta_{n,b} = \sum_{J \in \mathcal{S}} \mu_{n,b}^{(J)} \Big(\frac{f^{(J)}}{f_{n,b}^{(J)}} - 1\Big).
    \end{align*}
\end{lemma}

\begin{proof}[Proof of Lemma \ref{lemma:HFD_equivalence}]
    We first apply Lemma \ref{lemma:emp_hfd} to state the HFD $\smash{\{\tau_{n,b}^{(J)}\}_{J \in \mathcal{S}}}$ of $\tau_{n,b}$ with respect to $f_{n,b}$.
    Then, the proof exactly follows the proof of Lemma $4$ in TreeHFD article, where the theoretical discretized distribution $f_{\ell}$ is replaced by $f_{n,b}$, and all expectations are taken conditional on the training data $\Dn$.
\end{proof}

\begin{lemma} \label{lemma:bound_tree_HFD}
     If Assumption \ref{asm:three} and Specification \ref{spec:causaltreehfd} are satisfied, then there exists a constant $K >0$ such that for any tree $\tau_{n, b}$, we have
     \begin{align*}
         \sum_{J \in \mathcal{S}} \E[(\tau_{n, b}^{(J)}(\bX^{(J)}) - \mu_{n, b}^{(J)}&(\bX^{(J)}))^2 | \Dn]  \leq K \Delta_{n, f}^2 \E[\tau_{n, b}(\bX)^2  | \Dn].
     \end{align*}
\end{lemma}

\begin{proof}[Proof of Lemma \ref{lemma:bound_tree_HFD}]
    The proof follows the proof of Lemma $5$ in TreeHFD article, where $\tau_{n, b}^{(J)}$ is used instead of $\smash{\tau_{n, b}^{\star (J)}}$ thanks to Lemma \ref{lemma:HFD_equivalence}, and the empirical quantity $\Delta_{n, f}$ is used to bound the maximum difference between the density $f$ of $\bX$ and the empirical density $f_{n,b}$ within each cell of the Cartesian tree partitions.
\end{proof}

\begin{proof}[Proof of Corollary \ref{cor:convergence}]
    We apply Theorem \ref{thm:causaltreehfd_bound} to get for all $J \in \mathcal{S}$, 
    \begin{align*}
        \E[(\tau_n^{(J)}(\bX^{(J)})& - \tau^{(J)}(\bX^{(J)}))^2] \leq K_1 \E[( \tau(\bX) - \tau_n(\bX) )^2] \hspace*{4mm} + K_2 \E\Big[ \Delta_{n, f}^2 \Big( \frac{1}{B} \sum_{b=1}^B \sqrt{\E[\tau_{n, b}(\bX)^2|\Dn]}\Big)^2 \Big].
    \end{align*}
    Then, Theorem \ref{th:one} gives the convergence of causal forests, and the first term of the right hand side of the above inequality vanishes, since trees are bounded according to Specification \ref{spec:causaltreehfd}.
    
    The second term also vanishes, since trees are bounded, and by assumption, $\Delta_{n, f} \to 0$ in probability, which implies that $\E[\Delta_{n, f}^2] \to 0$, since $\Delta_{n, f}$ is bounded because $f$ is bounded from Assumption \ref{asm:three}, and $f_{n,b}$ is bounded from Specification \ref{spec:causaltreehfd}.

    For the second part of Corollary \ref{cor:convergence}, we first consider $\{\tau^{(J)}\}_{J \subset \mathcal{H}}$ the HFD of $\tau(\bX^{(\mathcal{H})})$ given by Theorem \ref{thm:hfd}. Then, we add null components for all subsets of variables of $S$ that are not included in $\mathcal{H}$, which satisfy the orthogonality conditions by construction. By uniqueness of the HFD, this augmented HFD is therefore the HFD of $\tau$ with respect to $\bX$, with $\tau^{(J)} = 0$ for all subsets $J \not \subset \mathcal{H}$.
    This last property combined with the first part of the corollary provides the second part of the corollary, that is, for all subsets $J \not \subset \mathcal{H}$,
    \begin{align*}
        \E[\tau_n^{(J)}(\bX^{(J)})^2] \xrightarrow[n \to \infty]{} 0.
    \end{align*}
\end{proof}

\end{document}